\documentclass[journal]{IEEEtran}

\newcommand{\W}{\mathcal{W}}
\usepackage{float}
\usepackage[ruled,vlined,linesnumbered]{algorithm2e}
\usepackage{yfonts}
\usepackage{url}

\usepackage[utf8]{inputenc}
\usepackage[T1]{fontenc}
\usepackage{ifthen}
\newboolean{PrintVersion}
\setboolean{PrintVersion}{false}
\usepackage{amsmath,amssymb,amstext,amsthm} 
\usepackage[pdftex]{graphicx} 
\usepackage{booktabs}
\usepackage{siunitx}
\usepackage{multirow}

\usepackage{color}
\usepackage{xcolor}

\usepackage[font=footnotesize]{caption}
\usepackage{subcaption}

\newtheorem{theorem}{Theorem}[section]
\newtheorem{proposition}[theorem]{Proposition}

\newtheorem{lemma}[theorem]{Lemma}

\newtheorem{definition}[theorem]{Definition}
\newtheorem{problem}[theorem]{Problem}
\newtheorem{example}[theorem]{Example}

\newtheorem{remark}[theorem]{Remark}

\newtheorem{assumption}[theorem]{Assumption}

\newcommand{\Ropt}{R_{\texttt{OPT}}}

\DeclareMathOperator*{\argmin}{arg\,min}

\title{\Large \textbf{Multi-Robot Persistent Monitoring:  Minimizing Latency \\ and Number of Robots with Recharging Constraints}}

\author{Ahmad Bilal Asghar, Shreyas Sundaram, and Stephen L. Smith
\thanks{This work was supported in part by the Natural Sciences and Engineering Research Council of Canada (NSERC)}
\thanks{A.\ B.\ Asghar is with the Maryland Robotics Center, University of Maryland, College Park MD (\texttt{abasghar@umd.edu}).  S.\ Sundaram is with the School of Electrical and Computer Engineering, Purdue University, West Lafayette, IN (\texttt{sundara2@purdue.edu}).  S.\ L.\  Smith is with the Department of Electrical and Computer Engineering, University of Waterloo, Waterloo, ON Canada (\texttt{stephen.smith@uwaterloo.ca}).}
}

\usepackage{lipsum}

\begin{document}
\maketitle
\begin{abstract}
In this paper we study multi-robot path planning for persistent monitoring tasks. We consider the case where robots have a limited battery capacity with a discharge time $D$. We represent the areas to be monitored as the vertices of a weighted
graph. For each vertex, there is a constraint on the maximum allowable
time between robot visits, called
the latency. The objective is to find the minimum number of
robots that can satisfy these latency constraints while also ensuring that the robots periodically charge at a recharging depot. The decision
version of this problem is known to be PSPACE-complete. We
present a $O(\frac{\log D}{\log \log D}\log \rho)$ approximation algorithm for the problem
where $\rho$ is the ratio of the maximum and the minimum latency
constraints. We also present an orienteering based heuristic
to solve the problem and show empirically that it typically provides higher quality solutions than the approximation algorithm. We extend our results to  provide an algorithm for the problem of minimizing
the maximum weighted latency given a fixed number of robots.
We evaluate our algorithms on large problem instances in a patrolling scenario and in a wildfire monitoring application. We also compare the algorithms with an existing solver on benchmark instances. 
\end{abstract}



\section{Introduction}
Due to rapid developments in mobile robotics, teams of robots can now perform long term monitoring tasks. Examples include infrastructure inspection~\cite{cabrita2010infrastructure} to detect anomalies or failures, patrolling for  surveillance~\cite{basilico2012patrolling,asghar2016stochastic}, 3D reconstruction of scenes~\cite{roberts2017submodular,bircher2015structural} in changing environments, informative path planning~\cite{cao2013multi} for observing dynamic properties of an area, and forest fire monitoring~\cite{merino2012unmanned}. The goal of these monitoring tasks is to deploy a team of cooperating mobile agents or robots in the dynamically changing environment to continually observe locations of interest. With limited resources, stationary agents cannot persistently monitor all regions of interest in a large environment, and therefore a team of mobile robots must patrol the environment to gather the information. If the environment evolves over time, as in persistent monitoring scenarios, then the locations in the environment need to be visited repeatedly by the team of robots.

In such tasks, locations of more importance should be visited more often as compared to locations with relatively low importance. This requirement can be modeled using a latency constraint for each location which specifies the maximum time the robots are allowed to stay away from a location. The locations that are at higher risk in a surveillance application will have a lower latency constraint and hence will be visited more often by the robots to satisfy that latency constraint. In this paper we consider the problem of finding the patrolling paths for the robots to satisfy the latency constraints of the locations in the environment. In applications, where the number of robots is constrained, finding a feasible solution to the problem may not be possible. In such cases, we consider the problem of minimizing the maximum weighted latency of the locations in the environment.

\begin{figure}[t]
    \centering
    \includegraphics[width=\linewidth]{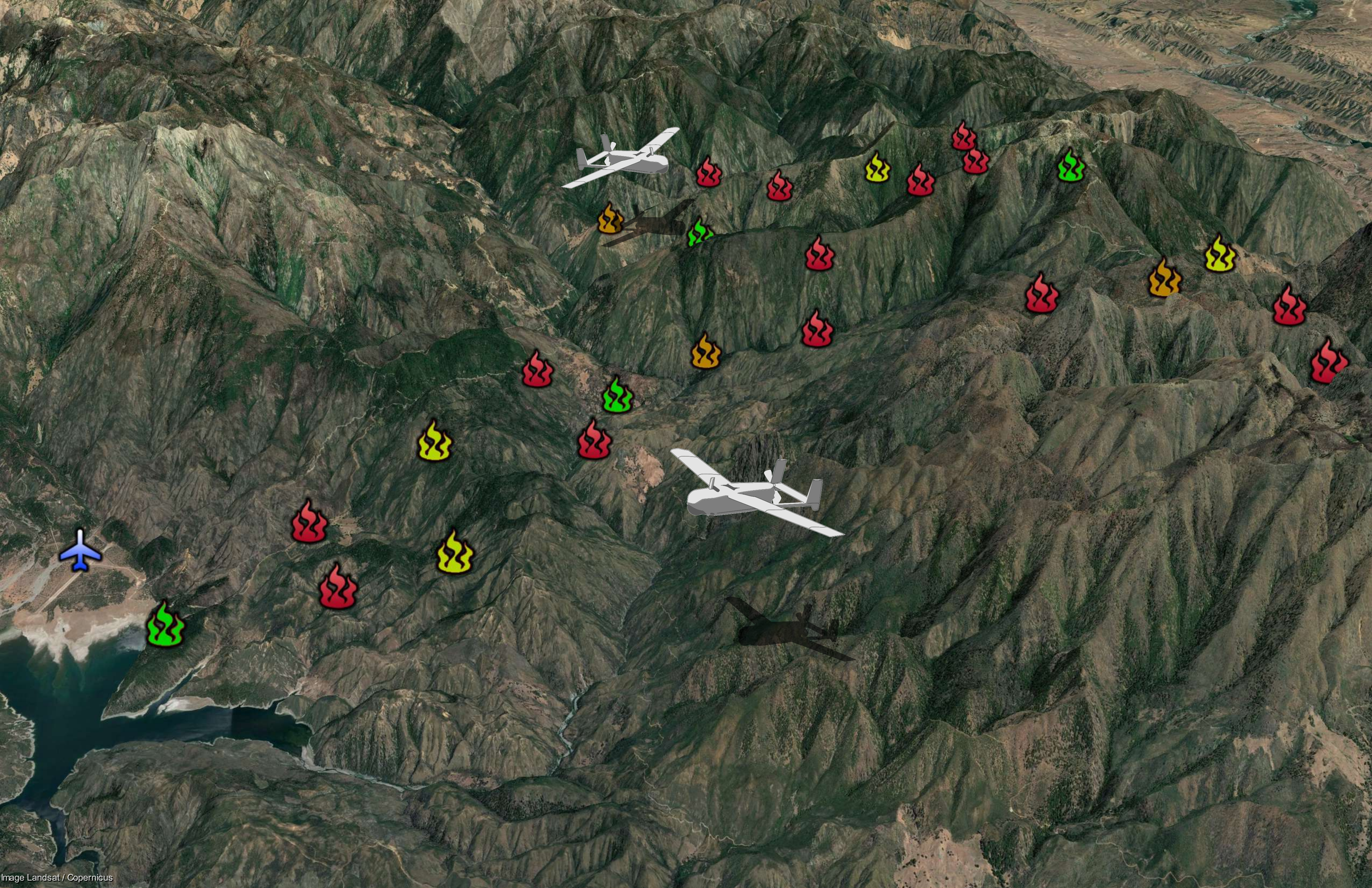}
    \caption{Two unmanned aerial vehicles (UAVs) surveying a wildfire-prone area for early detection and monitoring. The color-coded fire hotspot icons indicate the importance of a location based on the rate of spread and intensity of the fire estimated from satellite data. The UAVs must return to the nearby airstrip for refueling.}
    \label{fig:main}
\end{figure}

\subsection{Related Work}
\emph{Single robot patrolling paths:} Persistent monitoring problems have been extensively studied in the literature and there is a substantial amount of work dedicated to finding patrolling paths for the monitoring of environments~\cite{hokayem2007persistent,smith2011persistentocean,nigam2012control,elmaliach2009multi}. For a single robot, cyclic patrolling paths to detect randomly arriving events are studied in~\cite{yu2015persistent}. Different ways of randomizing a Travelling Salesman Path are considered in~\cite{Sak2008patrolling} to empirically show that randomizing the paths can help in detecting intrusions. In~\cite{Alamdari2014}, the problem of minimizing the maximum revisit time is considered and an approximation algorithm is proposed for the problem. Distributed algorithms to minimize refresh times are given in~\cite{pasqualetti2012cooperative}. Two heuristic approaches for minimizing the maximum staleness of information are compared in~\cite{nigam2008persistent}. In contrast to these works, we consider the multi-robot problem for minimizing the maximum revisit time and provide an algorithm with an upper bound on the cost of the solution.

\emph{Persistent coverage in continuous environments:} In~\cite{palacios2016multi}, persistent coverage using multiple robots in a continuous environment is considered. The coverage levels of locations decay with time and robots must keep visiting the locations to maintain a desired coverage level. This scenario of coverage quality deteriorating with time is also considered in~\cite{palacios2019equitable} where equitable partitions of a continuous non-convex environment are found, and then paths covering the partition for each robot are computed. In~\cite{palacios2016multi2}, the coverage of a location increases depending on the time a robot stays at that location, and the task is to find optimal staying times for robots.  
In this paper we consider a discrete environment represented as a graph, and instead of considering decaying coverage levels at the locations of the environment, the problem studied in this paper seeks to satisfy a given latency constraint for each location, which is a constraint on the maximum time a vertex can be left unvisited.

\emph{Monitoring with latency constraints:} The problem of finding paths to satisfy latency constraints has been studied in~\cite{las2013persistent,ho2015cyclic,drucker2016cyclic}. The authors in~\cite{las2013persistent} use incomplete greedy heuristics to determine if a single robot can satisfy the constraints on all vertices of a graph. They show that if a solution exists, then a periodic solution also exists. In this paper, we consider the multi-robot problem and our objective is to minimize the number of robots that can satisfy the latency constraints on the given graph. The multi-robot version of the problem has been considered in~\cite{drucker2016cyclic,ho2015cyclic}, where it is called \emph{Cyclic Routing of Unmanned Aerial Vehicles}. The decision version of the problem for a single robot is shown to be PSPACE-complete in~\cite{ho2015cyclic}. The authors also show that the length of even one period of a feasible walk  can be exponential in the size of the problem instance. In~\cite{drucker2016cyclic}, the authors propose a solver based on  Satisfiability Modulo
Theories (SMT). To apply an SMT solver, they impose an upper bound on the length of the period of the solution.  Since an upper bound is not known \emph{a priori}, the solver will not return the optimal solution if the true optimal period exceeds the bound.  The authors generate a library of test instances, but since their algorithm scales exponentially with the problem size, they solve instances up to only 7 vertices. We compare our algorithms with their solver and show that our algorithms run over $500$ times faster on average and return solutions with the same number of robots on $98\%$ of the benchmark instances provided by~\cite{drucker2016cyclic}.

\emph{Minimizing maximum weighted latency:} A closely related problem is where instead of latency constraints, each vertex has a weight associated with it and the objective is to minimize the maximum weighted latency (time between consecutive visits) for an infinite walk. We also propose an algorithm for the multi-robot version of this problem, along with recharging constraints. This problem of minimizing maximum weighted latency without recharging constraints has been studied in the literature. An approximation algorithm for the single robot version of the problem is provided in~\cite{Alamdari2014}. The authors in~\cite{afshani2022cyclic} consider the multi-robot version of this problem and present properties of cyclic solutions. They provide an approximation algorithm for the multi-robot problem in~\cite{afshani2021approximation}. 

\emph{Planning for energy constrained robots:} The robots performing persistent monitoring tasks in an environment stay active for very long times and therefore, they need to be refueled. The authors in~\cite{mathew2015multirobot} find routes for dedicated charging robots that rendezvous with the patrolling UAVs in order to replenish their batteries. This rendezvous problem with stochastic energy consumption is considered in~\cite{shi2022risk}. In~\cite{maini2018persistent} the problem of monitoring a terrain using heterogeneous robots with charging constraints is considered. The problem of scheduling spare drones to take place of the drones that need recharging is considered in~\cite{hartuv2018scheduling}. A persistent monitoring problem with a single UAV and single recharging depot is considered in~\cite{hari2022bounds} and an estimate on the number of locations visited within one charge is used to find a solution.

\emph{Timed-automaton based solutions:} Timed automata have been used to model general multi-robot path planning problems~\cite{quottrup2004multi} as the clock states can capture the concurrent time dependent motion. In~\cite{ulusoy2013optimality}, temporal logic constraints are used to specify high-level mission objectives to be achieved by a set of robots. The routing problem with latency constraints can also be modeled as a timed-automaton since multiple robots may require synchronization to satisfy the latency constraints. A timed automaton based solution to the problem is presented in~\cite{Drucker2014thesis}, however it is shown to perform more poorly than the SMT-based approach in~\cite{drucker2016cyclic}, which we use as a comparison for our proposed algorithms.

\subsection{Contributions}
The contributions of this paper are as follows. We introduce the problem of minimizing the number of robots to satisfy the latency and recharging constraints in Section~\ref{sec:latency-problem}. In Section~\ref{sec:latency-approx} we present an $O(\frac{\log D}{\log \log D}\log \rho)$ approximation algorithm for the problem where $D$ is the discharge time of the robots and $\rho$ is the ratio of the maximum and the minimum latency constraints. We provide several heuristic algorithms to solve the problem in Section~\ref{sec:latency-greedy} and show through simulations that an orienteering-based heuristic algorithm produces high-quality solutions. In Section~\ref{sec:latency-minmax} we study the problem of minimizing the maximum weighted latency using multiple robots and provide an algorithm for the problem by establishing a relationship to the problem of satisfying latency constraints. Finally, in Section~\ref{sec:latency-sim} we evaluate the performance of the algorithms on large problem instances in a patrolling scenario and in a wildfire monitoring application. We also show promising performance when comparing our algorithms against a state-of-the-art solver on benchmark instances.

This work builds on our preliminary conference paper~\cite{asghar2019multi}, which considered the monitoring problem with latency constraints, but without recharging.  In comparison to that earlier work we now provide complete proofs to establish approximation guarantees of our algorithms, we consider recharging constraints which substantially change the algorithms, guarantees and analysis, we provide an extension to min-max weighted latency, and we perform more extensive simulations.  

\section{Background}
In this section we present several key definitions that will be used throughout the paper.

A \emph{walk} in graph $G=(V,E)$ is defined as a sequence of vertices $(v_1,v_2,\ldots,v_k)$ such that $(v_i,v_{i+1}) \in E$ for each $1\leq i< k$.
A \emph{cycle} is a walk that starts and ends at the same vertex with no other vertex appearing more than once.
A \emph{tour} is a walk with no repeating edges.
An \emph{infinite walk} is an infinite sequence of vertices $(v_1,v_2,\ldots)$ such that $(v_i,v_{i+1})\in E$ for each $i\in \mathbb{N}$.

Given walks $W_1$ and $W_2$, $[W_1 ,W_2]$ represents the concatenation of the walks given that there is an edge between the last vertex of $W_1$ and the first vertex of $W_2$. Given a finite walk $W$, an infinite periodic walk is constructed by concatenating infinite copies of $W$ together, and is denoted by $\Delta (W)$.

In general, a walk can stay for some time at a vertex before traversing the edge towards the next vertex. Therefore we define a
\emph{timed walk} $W$ in graph $G$ as a sequence $(o_1,o_2,\ldots,o_k)$, where $o_i=(v_i,t_i)$ is an ordered pair that represents the holding time $t_i$ that the walk $W$ spends at vertex $v_i$, such that $(v_i,v_{i+1})\in E$ for each $1\leq i < k$. 

The definitions of infinite walk and periodic walk can be extended to infinite timed walk and periodic timed walk. A timed walk with ordered pairs of the form $(v_i,0)$ becomes a walk. The vertices traversed by walk $W$ are denoted by $V(W)$ and the length of walk $W=((v_1,t_1),(v_2,t_2),\ldots,(v_k,t_k))$ is denoted by $\ell(W) = \sum_{i=1}^{k-1}{\ell(v_i,v_{i+1})} + \sum_{i=1}^{k}{t_k}$.

Since we are considering multi-robot problems, synchronization between the walks is important. Given a set of walks $\W = \{W_1,W_2,\ldots, W_k\}$ on graph $G$, we assume that at time $0$, each robot $i$ is at the first vertex $v^i_1$ of its walk $W_i$, and will spend the holding time $t^i_1$ at that vertex before moving to $v^i_2$.\\

The \emph{Orienteering Problem}~\cite{golden1987orienteering} is a variant of the {\sc traveling salesman problem}. The input to the problem is a weighted directed or undirected graph $G=(V,E)$ with edge weights $\ell(e)$ for $e\in E$, two vertices $s,t \in V$, and a time limit $T_{max}$. Each vertex $i \in V$ has a score $\psi_i$ associated with it. The problem is to find an $s-t$ walk (a walk originating from $s$ and arriving eventually at $t$) of length not exceeding $T_{max}$ which maximizes the total score obtained by visiting the vertices. If a vertex is visited more than once in the walk, its score is counted only once.

A straightforward reduction from TSP renders the orienteering problem NP-hard. Chekuri \emph{et~al.}~\cite{chekuri2012improved} give a $(2+\epsilon)$ approximation algorithm for the problem. The existing exact and approximation algorithms are discussed and compared in~\cite{gunawan2016orienteering}.

\emph{Minimum Cycle Cover Problem:} 
Given a graph $G$ and length $\lambda$, the Minimum Cycle Cover Problem (MCCP) is to find minimum number of cycles that cover the whole graph such that the length of the longest cycle is at most $\lambda$. This problem is NP-hard and a $14/3$-approximation algorithm for MCCP is given in~\cite{yu2016improved}.

The following problem is a rooted version of MCCP.

\emph{Rooted Minimum Cycle Cover Problem:}
Given a graph $G=(\{V\cup \mu\},E)$ with $n$ vertices and a distance constraint $D$, the Rooted Minimum Cycle Cover Problem (RMCCP) is to find the minimum number of tours, each containing the depot $\mu$, such that the tours cover all the vertices in $G$ and the length of the longest tour is at most $\lambda$. This problem is also known as Distance Constrained Vehicle Routing Problem and a $O(\frac{\log D}{\log \log D})$ approximation algorithm for this problem with paths instead of tours is given in~\cite{friggstad2014approximation}. It is shown in~\cite{nagarajan2012approximation} that the approximation ratios of path and tour versions of Distance Constrained Vehicle Routing Problem are within a factor of $2$, resulting in a $O(\frac{\log D}{\log \log D})$ approximation algorithm for RMCCP. Using the Orienteering problem as a greedy subroutine results in a $O(\log n)$ approximation for RMCCP~\cite{nagarajan2012approximation}.

\section{Problem Statement}
\label{sec:latency-problem}
Consider an undirected weighted graph $G=(\{V\cup \mu\},E)$ representing the environment to be monitored by the robots. The set of vertices, $V$, represents the locations that need to be monitored and the vertex $\mu$ represents the depot or recharging location for the robots. The edge lengths are given by $\ell(e)$ for each edge $e\in E$. These edge lengths are metric and represent the time taken by the robots to travel between the vertices. Each vertex $v\in V$ has a latency constraint, denoted by $r(v)$, which represents the maximum time allowed between consecutive visits to that vertex. The robots have a discharging time of $D$, indicating that no robot can stay away from the recharging vertex $\mu$ for more than $D$ amount of time. 

In practice, robots may need to spend some time at a vertex to inspect it or gather information. This time can be added to the length of the edges leading to that vertex, resulting in an equivalent metric graph with zero stay times and modified latency constraints~\cite{drucker2016cyclic}. Hence, we assume without loss of generality, that the robots do not need to spend any time inspecting the vertices. Similarly, the recharging time can also be assumed to be zero.

Note that setting the latency constraint of the recharging vertex to $D$ does not capture the recharging constraints, as multiple robots can work together to satisfy the latency constraint of a vertex $v$ when $L(\W,v)\leq r(v)$, whereas the recharging constraint has to be satisfied for each robot, i.e., $L(W_k,\mu)\leq D$ for each robot $k$. 

We also make the following two assumptions about the problem instance.
\begin{assumption}
\label{assumptions}
We assume that
\begin{enumerate}
    \item $D\geq \max_v 2\ell(\{\mu,v\})$, and
    \item $r(v)\geq 2\ell(\{\mu,v\})$, $\forall v\in V$ \label{assumption2}
\end{enumerate}
\end{assumption}

The first assumption is necessary for the existence of a feasible solution while the second assumption is necessary for the feasibility of solutions that involve robots working independently so that a robot can recharge and visit a vertex while satisfying the latency constraint of that vertex by itself. We formally define the problem statement after the following definition of latency.

\begin{definition}[Latency]
Given a set of infinite walks $\W = \{W_1,W_2,\ldots, W_k\}$ on a graph $G$, let $a^v_i$ represent the $i^{th}$ arrival time for the walks to vertex $v$. Similarly, let $d^v_i$ represent the $i^{th}$ departure time from vertex $v$. Then the latency $L(\W,v)$ of vertex $v$ on walks $\W$ is defined as the maximum time spent away from vertex $v$ by the walks, i.e., $L(\W,v) = \sup_{i}{(a^v_{i+1}-d^v_i)}$.
\end{definition}

\begin{problem}[Minimizing Robots with Latency and Recharging Constraints]
\label{pbm:min_robs} Given an undirected graph $G=(\{V\cup \mu\})$, discharging time $D$, and latency constraints $r:V\to \mathbb{R}\cup\{\infty\}$, find a set of walks $\W$ with minimum cardinality such that:
\begin{enumerate}
    \item the latency constraints of all vertices are satsified, i.e., $L(\W,v) \leq r(v), \forall v \in V$, and
    \item no robot runs out of charge (spends more than $D$ time away from $\mu$).
\end{enumerate}

\end{problem}
The decision version of the problem is to determine whether there exists a set of $R$ walks $\W=\{W_1,W_2,\ldots,W_R\}$ such that $L(\W,v) \leq r(v)$ for all $v \in V$ and $L(W_k,\mu)\leq D$ for all $k\in\{1,\ldots,R\}$. Note that although we have assumed that the robots do not need to stay at a vertex to inspect that vertex, in a general solution to the problem the robots might still need to stay at vertices in order to coordinate among themselves to satisfy the latency constraints. Therefore a general solution to Problem~\ref{pbm:min_robs} will be a set of timed walks with possibly non-zero holding times.

This problem is defined on a graph, where the edges and their lengths represent the movement of robots within the environment. The graph can be generated using a method such as probabilistic roadmap (PRM) or any other environment decomposition method ~\cite{kavraki1996probabilistic}. The latency constraints specify the maximum amount of time that can pass between visits to a vertex. For instance, in dynamic scene reconstruction, each vertex corresponds to a viewpoint~\cite{roberts2017submodular}. The latency constraints may indicate the maximum staleness of information that can be tolerated for the voxels captured from that viewpoint.

\subsection{Multiple Robots on the Same Walk}\label{sec:factor_R}
In multi-robot problems that involve finding cycles or tours in a graph, the cost of the tour can be reduced by a factor of $n$ by placing $n$ robots on the tour such that each robot follows the one ahead of it at a distance of $1/n$ times the length of the tour~\cite{chevaleyre2004theoretical}. We will refer to this placement of multiple robots on a tour as equally spacing robots on a tour. Equally spacing multiple robots on the solution for a single robot does not work in a similar manner for Problem~\ref{pbm:min_robs}: if a periodic walk $W$ gives latency $L(W,v)$ on vertex $v$, equally spacing more robots on one period of that walk does not necessarily reduce the latency for that vertex. Figure~\ref{fig:example1} gives an example of such a walk. The latency of vertices $a,b$ and $c$ on the walk $(a,b,a,c,a)$ are $2,4$ and $4$ respectively. The length of one period of the walk is $4$. If we place another robot following the first robot with a lag of $2$ units, the latency of vertex $a$ remains the same. If we place the second robot at a lag of $1$ unit, the latency will reduce to $1$ for vertex $a$ and $3$ for vertices $b$ and $c$.
Hence, cycles are an exception, for general walks we need more sophisticated algorithms than finding a walk for a single robot and adding more robots on that walk until the constraints are satisfied.
\begin{figure}[t]
\centering
\includegraphics[width=.4\linewidth]{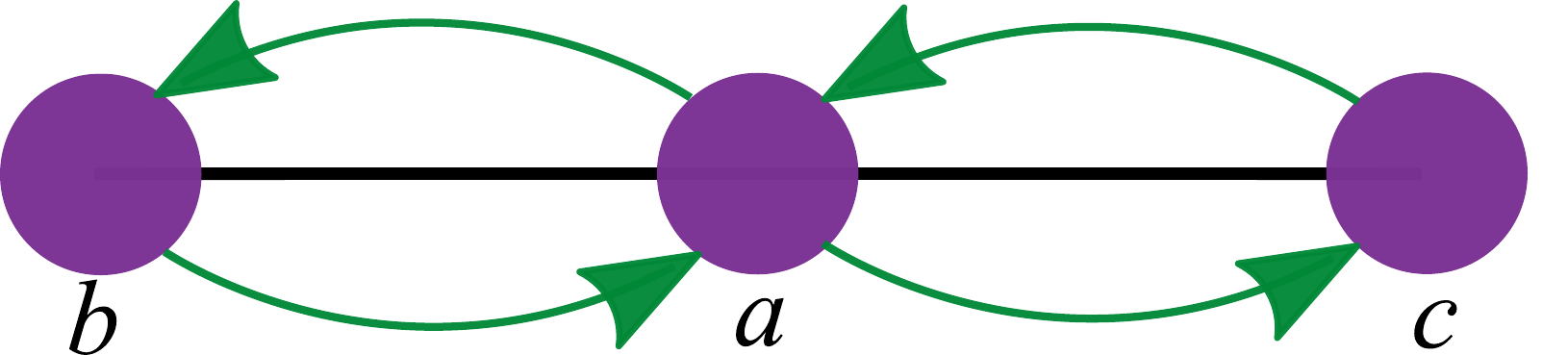}
\caption[Equally spacing robots on a walk]{A graph with three vertices and the walk $(a,b,a,c,a)$. The length of shown edges is one. Equally spacing two robots on this walk does not halve the latencies.}
\label{fig:example1}
\end{figure}

\section{Approximation Algorithm}
\label{sec:latency-approx}

The decision version of Problem~\ref{pbm:min_robs} with a single robot and without recharging constraints has been shown to be PSPACE-complete in~\cite{ho2015cyclic}. As a result, approximate and heuristic solutions must be employed to solve the problem. In this section, we present an approximation algorithm for the problem.

We will begin by discussing a simple approach to the problem and then gradually improve upon it to derive the approximation algorithm. One naive solution to the problem without recharging constraint is to find a TSP tour of the graph and equally space robots on that tour to satisfy all the latency constraints. However, a single vertex with a very small $r(v)$ can result in a solution with the number of robots proportional to $1/r(v)$. To solve this issue, we can partition the vertices of the graph such that the latencies in one partition are close to each other, and then place robots on the TSP tour of each partition. If more than one robot is required for a partition $V'$, we can solve the MCCP for that partition. The benefit of using the MCCP over placing multiple robots on a TSP is that if all the vertices in $V'$ had the same latency requirement, then we have a guarantee on the number of cycles required for that partition. To handle the recharging constraint, we can use RMCCP instead of MCCP on a partition to get minimum number of cycles rooted at $\mu$. 

However, a general solution to the problem might not consist of simple cycles. Lemma~\ref{lem:yann} establishes a connection between solutions made up of cycles and general solutions, and demonstrates that a solution composed of cycles will have latencies that are no more than twice as long as any general solution with the same number of robots. Therefore, by solving the RMCCP on a partition with its latency constraints multiplied by two, we can upper bound the number of cycles required. We can then use the idea from TSP based solutions and assign multiple robots to each cycle to meet the latency constraints.

The approximation algorithm is given in Algorithm~\ref{alg:approx_latency}. The first five lines of the algorithm partition the vertices according to their latency constraints. For a partition $V_i$, the function RMCCP$(V_i, \mu, D)$ called in Line~\ref{algln:MCCP} uses an approximation algorithm for the Rooted Minimum Cycle Cover Problem to find cycles rooted at $\mu$ such that each cycle has length at most $D$. Then, those cycles are assigned to different robots to satisfy the latency constraints. We use the following definition to establish the approximation ratio of Algorithm~\ref{alg:approx_latency}. A similar relaxation technique was also used in~\cite{Alamdari2014}.

\begin{algorithm}

\DontPrintSemicolon
\caption{\textsc{ApproximationAlgorithm}}
\label{alg:approx_latency}

\KwIn{Graph $G=(\{V\cup\mu\},E)$, discharging time $D$, latency constraints $r(v), \forall v\in V$}
\KwOut{A set of walks $\W=\{W_1,W_2,\ldots,W_R\}$, such that $L(\W,v)\leq r(v)$, $\forall v\in V$ and $L(W_k,
    \mu)\leq D$, $\forall k\in \{1,\ldots,R\}$}
\vspace{0.2em}
\hrule
 $r_{\texttt{max}}\gets \max_v r(v)$\;
 $r_{\texttt{min}}\gets \min_v r(v)$\;
 \eIf{$r_{\texttt{max}}/ r_{\texttt{min}}$ is an exact power of $2$}{$\rho \gets \frac{r_{\texttt{max}}}{ r_{\texttt{min}}}+1$\;}
 {$\rho \gets \frac{r_{\texttt{max}}}{ r_{\texttt{min}}}$\;}
$\W = \{\}$\;
Let $V_{i}$ be the set of vertices $v$ such that  $ r_{\texttt{min}}2^{i-1}\leq r(v) < r_{\texttt{min}}2^{i} $  for $1\leq i \leq \lceil \log_2 \rho \rceil $\label{algln:levels}\;
\For {$i=1$ to $\lceil \log_2 \rho \rceil$}{
$\{C_1,\ldots,C_{|\mathcal{C}|}\} = $ RMCCP$(V_{i}, \mu, D)$
	\label{algln:MCCP}\;
 Let $b = \max\{1,\lfloor ( r_{\texttt{min}}2^{i+1}/D ) \rfloor\}$\;
 \For {$j = 0$ to $\lfloor\frac{|\mathcal{C}|}{b}\rfloor -1 $}{
 Concatenate $\{C_{jb+1},\ldots,C_{(j+1)b}\}$ cycles to get walk $W'$\;
 Equally space $\lceil \ell(W')/\min_{v\in V(C)}{r(v)} \rceil$ robots on walk $W'$ to get walks $\W''$\label{algln:final}\;
 $\W=\{\W,\W''\}$\;
 }
}
\end{algorithm}

\begin{definition}[Relaxed Latency Constraints]
Let $r_{\texttt{min}}=\min_v r(v)$. The latency constraints of the problem are said to be relaxed if for every vertex $v$, its latency constraint is updated from $r(v)$ to $\bar{r}(v) = r_{\texttt{min}}2^{x}$ such that $x$ is the smallest integer for which $r(v) < r_{\texttt{min}}2^{x}$.
\end{definition}
We will also need the following lemma that is an extension of Lemma 2 in~\cite{chevaleyre2004theoretical}.
\begin{lemma}
\label{lem:yann}
Consider an undirected metric graph $G=(\{V\cup\mu\},E)$ with latency constraint $r(v) = M$ for all $v\in V$, and recharging depot $\mu$ with discharging time $D$. For any set of walks $\W=\{W_1,\ldots,W_R\}$ that satisfies the recharging and latency constraints, there exists a set of walks $\W'$ satisfying the recharging constraints such that $|\W|=|\W'|$, and each walk $W_i\in \W'$ is of the form $[C_1,\ldots,C_a]$, $a\geq 1$, where $C_j$ is a cycle rooted at $\mu$. Moreover, $\max_{v\in V} L(\W',v)\leq 2M$.  
\end{lemma}

\begin{proof}
Let $v_1^k,\ldots,v_{m_k}^k$ be the ordering of the vertices visited by walk $W_k$ between time $t=0$ and time $t=M$. Since $\W$ satisfies the latency constraints, $\bigcup_{k=1}^R \{v_1^k,\ldots,v_{m_k}^k\} $ contains all the vertices in $V$. 

Let $W'_k$ be a periodic walk that visits vertices $v_1^k,\ldots,v_{m_k}^k,\mu$ in one period. Since $W_k$ satisfies recharging constraints, $W'_k$ also satisfies the recharging constraints due to the metric property of the graph. Also, if any vertex is visited more than once between two consecutive visits to $\mu$ in $W'_k$, we can shortcut all the instances of that vertex apart from one in order to get cycles rooted at $\mu$. 

The length of the segment $v_1^k,\ldots,v_{m_k}^k$ of the walk $W'_k$  is at most $M$ and by Assumption~\ref{assumptions}-\ref{assumption2}, $\ell({v_{m_k}^k,\mu}) +\ell(\{\mu,v_1^k\})\leq M$. Hence, the latency of all the vertices in $V$ is at most $2M$.
\end{proof}

The following proposition gives the approximation factor of Algorithm~\ref{alg:approx_latency}.
\begin{proposition}
\label{prop:approx}
Given an undirected metric graph $G=(\{V\cup \mu\},E)$ with latency constraints $r(v)$ for $v\in V$, and discharging time $D$, Algorithm~\ref{alg:approx_latency} constructs $R$ walks $\W=\{W_1,W_2,\ldots, W_R\}$ such that
\begin{enumerate}
    \item $L(\W,v)\leq r(v)$ for all $v\in V$,
    \item $L(W_k,\mu)\leq D$, $\forall k\in \{1,\ldots,R\}$, and
    \item $R = O(\min\{\log n ,\frac{\log D}{\log \log D} \}\lceil\log(\rho)\rceil) \Ropt$,
\end{enumerate}
where $\Ropt$ is the minimum number of robots required to solve the Problem~\ref{pbm:min_robs}.
\end{proposition}

\begin{proof}
Given that $\Ropt$ robots will satisfy the latency constraints $r(v)$, they will also satisfy the relaxed constraints $\bar{r}(v)$ since $\bar{r}(v)>r(v)$. Therefore, there exists a set of at most $\Ropt$ walks $\W^*$ such that for $v\in V_i$, $L(\W^*,v)\leq r_{\texttt{min}}2^{i}$, and each walk satisfies the recharging constraint. 

Using Lemma~\ref{lem:yann}, given the set $\W^*$, a set of $\Ropt$ walks can be constructed in $V_i$ such that the latency of each vertex in $V_i$ is at most $r_{\texttt{min}}2^{i+1}$ and the recharging constraint is satisfied for each robot. Moreover, each walk in this set is a concatenation of cycles of length at most $D$ rooted at $\mu$. Note that the maximum number of cycles assigned to a walk is $\lceil ( r_{\texttt{min}}2^{i+1}/D ) \rceil$. Hence, running an $\alpha$ approximation algorithm for Rooted Minimum Cycle Cover Problem (RMCCP) on the subgraph with vertices $V_i$ and with maximum cycle length $D$ will not return more than $\alpha \lceil ( r_{\texttt{min}}2^{i+1}/D ) \rceil \Ropt$ rooted cycles.

We can assign $\max\{1,\lfloor ( r_{\texttt{min}}2^{i+1}/D ) \rfloor\}$ cycles to each walk such that the latency of each vertex in the assigned vertices is at most $r_{\texttt{min}}2^{i+1}$. Therefore, we can assign all the $\alpha \lceil ( r_{\texttt{min}}2^{i+1}/D ) \rceil \Ropt$ cycles to $2\alpha \Ropt$ walks. Since each walk is independent of other walks and is a concatenation of cycles, equally spacing $k$ robots within a period of each such walk will reduce the latency of each vertex on that walk by a factor of $k$. As $r(v) \geq  r_{\texttt{min}}2^{i+1}/4$ for each $v\in V_i$, we will need to place at most $4$ robots on each such walk to satisfy the latency constraints.

Finally, since there are at most $\lceil\log \rho\rceil$ partitions, we will need $8\alpha \lceil\log(\rho)\rceil \Ropt$ robots. As the approximation ratio for RMCCP is $O(\min\{\log n ,\frac{\log D}{\log \log D}\})$, the algorithm will return $R \leq O(\min\{\log n,\frac{\log D}{\log \log D}\} \lceil\log(\rho)\rceil) \Ropt$ walks.
\end{proof}

\textbf{Runtime:} Since we run the approximation algorithm for RMCCP on partitions of the graph, Algorithm~\ref{alg:approx_latency} has the same time complexity as that of the approximation algorithm of RMCCP. That is because the runtime of RMCCP is superlinear, so if $\sum |V_i| = |V|$, then $\sum |V_i|^p \leq |V|^p $ for $p\geq 1$.

\subsection{Infinite Discharge Time}
For the case when the robots do not need to recharge, for example in situations where the discharge time of the robots is greater than the total mission length, we can improve the approximation ratio by using the Minimum Cycle Cover Problem instaed of its rooted version. The Line~\ref{algln:MCCP} of Algorithm~\ref{alg:approx_latency} can be replaced by MCCP$(V_i,r_{\texttt{min}}2^{i+1})$ to get a set of cycles in partition $V_i$ such that the length of the longest cycle is at most $r_{\texttt{min}}2^{i+1}$. Multiple robots can then be placed on each of the resulting cycles as in Line~\ref{algln:final} of the algorithm to get the solution. The approximation ratio of this solution is given below.

\begin{proposition}
\label{prop:approx_no_recharge}
Given an undirected metric graph $G=(V,E)$ with latency constraints $r(v)$ for $v\in V$, there exists an algorithm that constructs $R$ walks $\W=\{W_1,W_2,\ldots, W_R\}$ such that $L(\W,v)\leq r(v)$ for all $v\in V$ and $R \leq 4\alpha\lceil\log(\rho) \rceil\Ropt$, where $\Ropt$ is the minimum number of robots required to satisfy the latency constraints and $\alpha$ is the approximation factor of MCCP.
\end{proposition}

The proof of this proposition follows from the proof of Proposition~\ref{prop:approx}.

\begin{remark}[Heuristic Improvements]
Instead of finding cycles using MCCP for each partition $V_i$, we can also equally space robots on the Traveling Salesman Tour of $V_i$ to get a feasible solution. In practice, we use both of these methods and pick the solution that gives the lower number of robots for each $V_i$. This modification can return better solutions to the problem but does not improve the approximation guarantee established in Proposition~\ref{prop:approx_no_recharge}.
\end{remark}

\section{Heuristic Algorithms}
\label{sec:latency-greedy}
The approximation algorithm for Problem~\ref{pbm:min_robs} presented in Section~\ref{sec:latency-approx} is guaranteed to provide a solution within a fixed factor of the optimal solution. In this section, we propose a heuristic algorithm based on the Orienteering Problem, which in practice provides high-quality solutions.
\subsection{Partitioned Solutions}
In general, walks in a solution of the problem may share some of the vertices. However, sharing the vertices by multiple robots requires coordination and communication among the robots. Such strategies may also require the robots to hold at certain vertices for some time before traversing the next edge, in order to maintain synchronization. This presents difficulties for vehicles that must maintain forward motion, such as fixed-wing aircraft. The following example illustrates that if vertices are shared by the robots, lack of coordination or perturbation in edge weights can lead to large errors in latencies.

\begin{example} Consider the problem instance without recharging constraints shown in Figure~\ref{fig:example2}. An optimal set of walks for this problem is given by $\{W_1,W_2,W_3\}$ where $W_1 = ((a,1),(b,1))$, $W_2=((b,0),(c,0))$ and $W_3 = ((c,0),(d,1),(c,1))$. Note that walk $W_1$ starts by staying on vertex $a$, while $W_2$ leaves vertex $b$ and $W_3$ leaves vertex $c$. If the length of edge $\{b,c\}$ changes from $3$ to $3-\epsilon$, (e.g., if the robot's speed increases slightly) the latencies of vertices $b$ and $c$ will keep changing with time and will go up to $5$. Hence, a small deviation in robot speed can result in a large impact on the monitoring objective. Also note that any partitioned solution will need $4$ robots. 
\begin{figure}[ht]
\centering
\includegraphics[width=.6\linewidth]{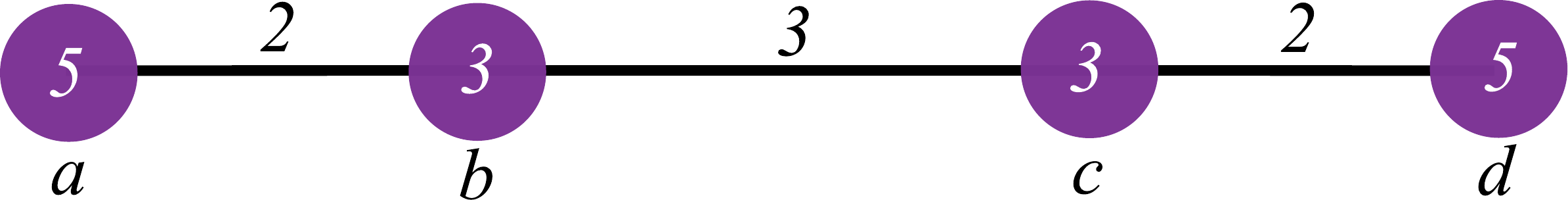}
\caption[A problem instance with an optimal set of walks that share vertices.]{A problem instance with an optimal set of walks that share vertices. The latency constraints for each vertex are written inside that vertex. The edge lengths are labeled with the edges. The optimal walks are $\{W_1,W_2,W_3\}$ where $W_1 = ((a,1),(b,1))$, $W_2=((b,0),(c,0))$ and $W_3 = ((c,0),(d,1),(c,1))$.}
\label{fig:example2}
\end{figure}
\end{example}
Since the above mentioned issues will not occur if the robots do not share the vertices of the graph, and the problem is PSPACE-complete even for a single robot, we focus on finding partitioned walks in this section. The general approach used in this section is to find a single walk that satisfies latency constraints on a subset of vertices $V'\subseteq V$. Note that we do not know $V'$ beforehand, but a feasible walk on a subset of vertices will determine $V'$. We then repeat this process of finding feasible walks on the remaining vertices of the graph until the whole graph is covered.

\subsection{Greedy Algorithm}
\label{sec:simple_greedy}
We now consider the problem of finding a single walk on the graph $G=(\{V\cup \mu\})$ that satisfies the latency constraints on the vertices in $V'\subseteq V$. Given a robot walking on a graph, let $p(k)$ represent the vertex occupied by the robot after traversing $k$ edges (after $k$ steps) of the walk. Also, at step $k$, let $c(k)$ represent the current time to discharge, and let the maximum time left until a vertex $i$ has to be visited by the robot for its latency to be satisfied be represented by $s_i(k)$. If that vertex is not visited by the robot within that time, we say that the vertex expired. Hence, the vector $s(k)=[s_1(k),\ldots,s_{|V'|}(k)]^T$ represents the time to expiry for vertices in $V'$. The walk starts form the depot $\mu$ and at the start of the walk, $s_i(0)=r(i)$ and $c(0)=D$. The values $s_i(k)$ and $c(k)$ evolve according to the following equations:
\begin{equation}
\label{eq:slackupadate}
s_i(k) = \begin{cases}
r(i) \quad  &\quad \text{if } p(k)=i \\
s_i(k-1)- \ell(p(k-1),p(k)) & \quad \text{otherwise}. \\
\end{cases}
\end{equation}

\begin{equation}
c(k) = \begin{cases}
D \quad  &\quad \text{if } p(k)=\mu \\
c(k-1)- \ell(p(k-1),p(k)) & \quad \text{otherwise}. \\
\end{cases}
\end{equation}

As we are considering the problem of finding a walk for a single robot, the recharging and latency constraints are $L(W,\mu)\leq D$ and $L(W,v)\leq r(v)$ for all $v$, respectively. Since $s_i(k)$ and $c(k)$ evolve identically if we set $r(\mu) = D$ and $s_\mu(k) = c(k)$, we can deal with the recharging constraint as another latency constraint. We will use the notation $s_i$ without the step $k$ if it is clear that we are talking about the current time to expiry. 

An incomplete greedy heuristic for the decision version of the problem with $R=1$ and no recharging constraints is presented in~\cite{las2013persistent}. The heuristic is to pick the vertex with minimum value of $s_i(k)$ as the next vertex to be visited by the robot. This heuristic does not ensure that all the vertices on the walk will have their latency constraints satisfied since the distance to a vertex $i$ to be visited might get larger than $s_i(k)$. We propose a modification to the heuristic in order to get feasible solutions. Given a walk $W$ on graph $G$, a function \textsc{PeriodicFeasibility$(W,G)$} determines whether the periodic walk $\Delta(W)$ is feasible on the vertices that are visited by $W$. This can be done simply in $O(|W|)$ time by traversing the walk $[W, W]$ and checking if the time to expiry for any of the visited vertices becomes negative. Given this function, the greedy algorithm is to pick the vertex $i=\argmin\{s_j\} $ subject to the constraint that \textsc{PeriodicFeasibility$([W , i],G)$} returns true, where $W$ is the walk traversed so far. The algorithm terminates when all the vertices are either expired, or covered by the walk. An example of a step of the greedy algorithm is depicted in Figure~\ref{fig:latency-greedy}.
\begin{figure}
\centering
  \includegraphics[width=.4\linewidth]{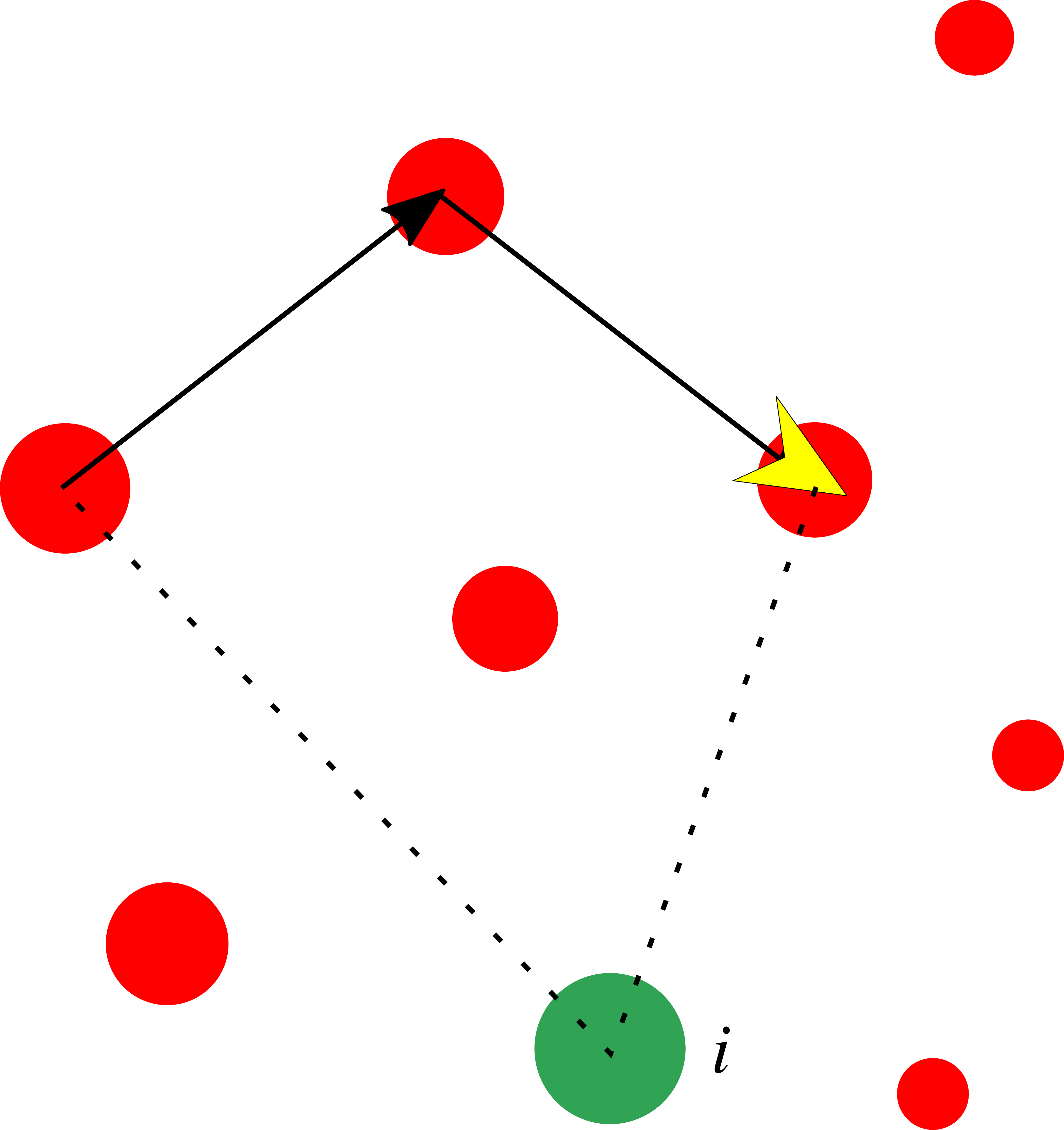}
\caption[Example of a step of the greedy algorithm]{An example depicting a step of the greedy algorithm. The solid black lines show the existing walk $W$, and the yellow arrow shows the current position of the robot. The green vertex $i$ is the vertex with least time to expiry. That vertex is appended to the existing walk and the closed walk is checked for feasibility before appending $i$ to $W$.}
\label{fig:latency-greedy}
\end{figure}

\subsection{Recursive Greedy Algorithm}
The greedy algorithm presented above selects the vertex with the minimum time to expiry and adds it to the current walk if the resulting walk is feasible. In this section we extend the greedy algorithm and instead of going directly to vertex $i=\argmin\{s_j\}$, we check if other vertices can be visited on the way to vertex $i$. This is done greedily as well, where we pick the vertex $j\neq i$ with the minimum time to expiry such that the walk $[W,j,i]$ remains feasible on the vertices visited by the walk. This is done recursively until no more vertices can be added to the walk. We refer to this algorithm as \textsc{RecursiveGreedy}.

\subsection{Orienteering Based Greedy Algorithm}
In this section we take the idea of the \textsc{RecursiveGreedy} algorithm one step further and try to visit the `best' combination of vertices on the way to vertex $i=\argmin\{s_j\}$. This algorithm also finds partitioned walks like the previous two heuristic algorithms by finding a feasible walk on a subset of vertices and then considering the remaining subgraph. The algorithm is presented in Algorithm~\ref{alg:greedy}. From the current vertex $x$, the target vertex $y$ is picked greedily as described in Section~\ref{sec:simple_greedy}. Then the time $d$ is calculated in Line~\ref{algln:dist} which is the maximum time to go from $x$ to $y$ for which the periodic walk remains feasible. In Line~\ref{algln:app_walk}, \textsc{Orienteering}$(V-V_{\texttt{exp}},x,y,d,\psi)$ finds a path in the vertices $V-V_{\texttt{exp
}}$ from $x$ to $y$ of length at most $d$ maximizing the sum of the weights $\psi$ on the vertices of the path. The set $V_{\texttt{exp
}}$ represents the expired vertices whose latencies cannot be satisfied by the current walk, and they will be considered by the next robot. The vertices with less time to expiry are given more importance in the path by setting weight $\psi_i=1/s_i$ for vertex $i$. The vertices that are already in the walk will remain feasible, and so their weight is discounted by a small number $m$ to encourage the path to explore unvisited vertices. One step of the algorithm is depicted in Figure~\ref{fig:latency_orienteering_example}. The following result shows that this algorithm will always find a feasible solution.
\begin{figure}
\centering
\begin{minipage}[t]{.475\columnwidth}
  \centering
  \includegraphics[width=.8\linewidth]{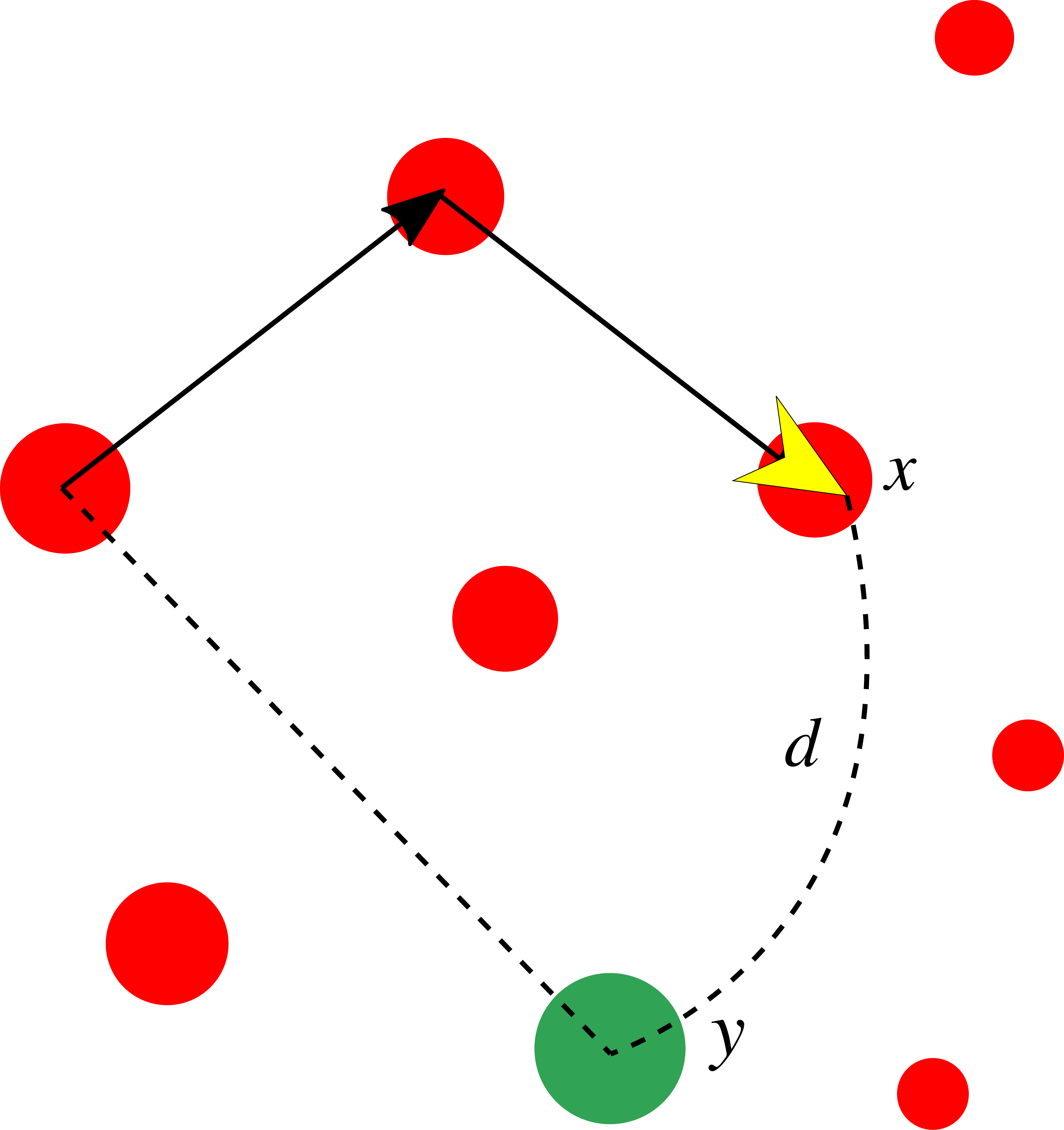}
\end{minipage}\hfill
\begin{minipage}[t]{.475\columnwidth}
  \centering
  \includegraphics[width=.8\linewidth]{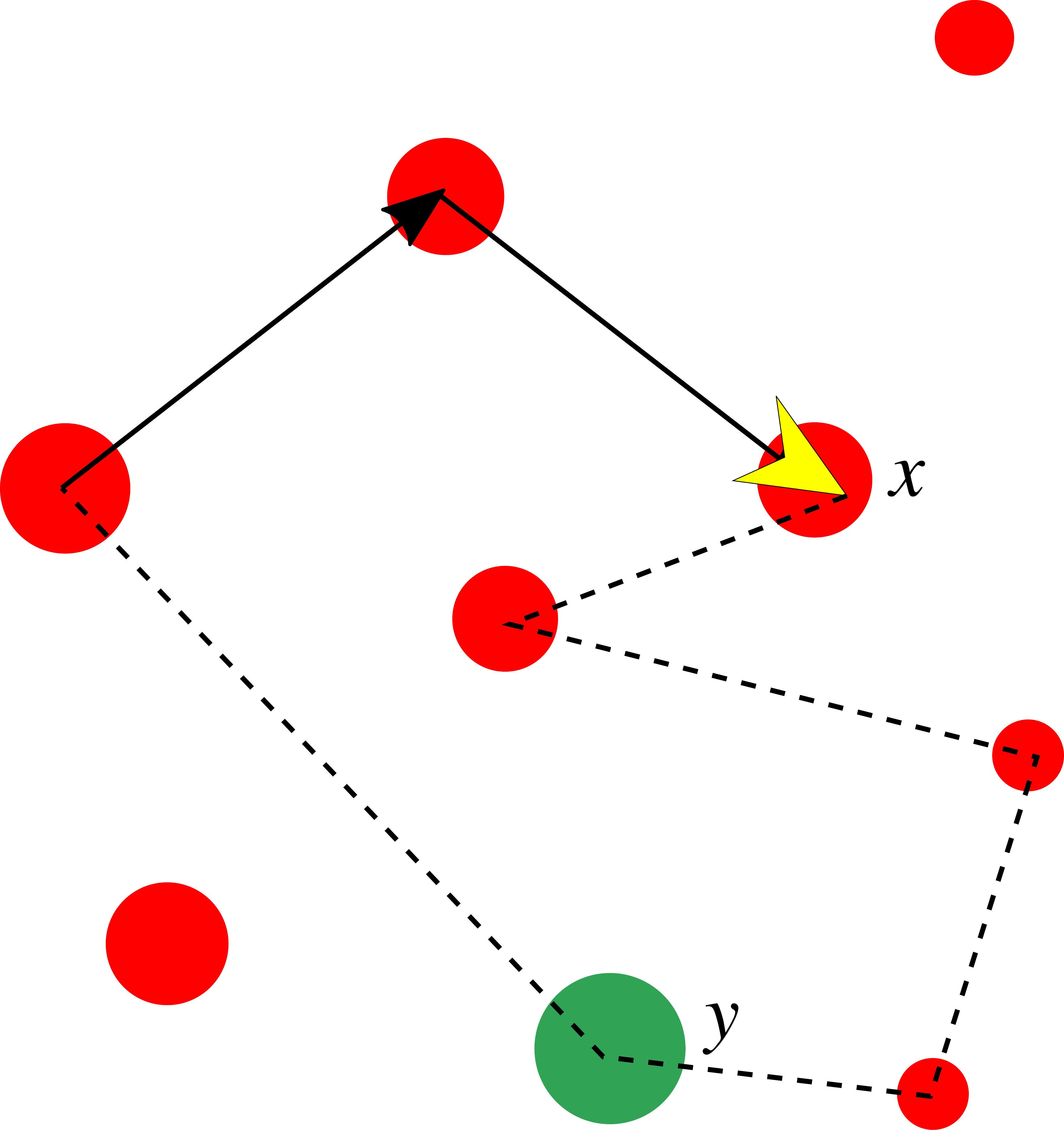}
\end{minipage}
\caption[Orienteering based greedy algorithm]{A step of the \textsc{OrienteeringGreedy} algorithm. The solid lines represent the walk traversed so far. Yellow arrow is the current position of the robot. Vertex $y$ is picked greedily. The figure on the left shows the time $d$ that can be spent before going to $y$. The figure on the right shows the orienteering path from $x$ to $y$.}
\label{fig:latency_orienteering_example}
\end{figure}

\begin{algorithm}
 
 \DontPrintSemicolon
\caption{\textsc{OrienteeringGreedy}}
\label{alg:greedy}

	\KwIn{ Graph $G=(\{V\cup \mu\},E)$, discharging time $D$, latency constraints $r(v), \forall v$}
	\KwOut{ A set of $R$ walks $\W$, such that $L(\W,v)\leq r(v)$, $\forall v\in V$ and $L(W_k,\mu)\leq D$, $\forall k\in \{1,2,\ldots,R\}$ }
  \vspace{0.2em}
  \hrule
	$j=1$, $\W = \{\}$\;
	\While {$V$ is not empty}
 {
	$V_{\texttt{exp}}=\{\}$\;
	$r(\mu) =D$\;
	$s_i = r(i)$ for all $i\in \{V\cup \mu\}$\;
	$W_j = \mu$\;
	\While {$V-V(W_j)-V_{\texttt{exp}}$ is not empty\label{algln:while}}
 {
 
	$x = $ last vertex in $W_j$\;
	\For {$y\in V\cup \mu-V_{\texttt{exp}}$ \textup{in increasing order of $s$}}
 
 {
	\eIf {\textsc{PeriodicFeasibility$([W_j , y],G)$}}
 {
	Use binary search between $\ell(x,y)$ and $s_y$ to get $d$ (time to go from $x$ to $y$) such that $[W_j,y]$ remains feasible\; \label{algln:dist}
	\For {$z$ in $V-(V_{\texttt{exp}}\cup V(W_j))$}
 {
	\If {$s_z < d + \ell(y,\mu)$} {$V_{\texttt{exp}}=V_{\texttt{exp}}\cup z$\; \label{algln:newif}}
	}

	$\psi_i=1/s_i$ for all non expired vertices $i$\;
	$\psi_i = m \psi_i$ for $i$ in $V(W_j)$\;
	$W_j = [W_j, $\textsc{Orienteering}$(V-V_{\texttt{exp}},x,y,d,\psi)]$\; \label{algln:app_walk}
	Update $s$ using Equation~(\ref{eq:slackupadate})\;
 \textbf{break}\;
 }
	{
	 $V_{\texttt{exp}}=V_{\texttt{exp}}\cup y$\;
	}
	}
 
 }
	 $\W = \{\W,W_j\}$\;
  $j = j+1$\; \label{algln:sol_set}
	$V = V-V(W_j)$\; \label{algln:reduce_size}

}
\end{algorithm}

\begin{proposition}
Algorithm~\ref{alg:greedy} returns a feasible solution, i.e., for the set of walks $\W$ returned by Algorithm~\ref{alg:greedy}, $L(\W,v)\leq r(v)$, for all $v\in V$, and $L(W_k,\mu)\leq D$ for each robot $k$.
\end{proposition}
\begin{proof}
The vertices covered by the walk $W_j$ added to the solution in Line~\ref{algln:sol_set} are removed from the set of vertices before finding the rest of the walks. Hence the latencies of the vertices $V(W_j)$ will be satisfied by only $W_j$. We will show that every time $W_j$ is appended in Line~\ref{algln:app_walk}, it remains feasible on $V(W_j)$.

$W_j$ starts from the vertex $\mu$, and hence is feasible at the start. Let us denote $W_j^-$ as the walk before Line~\ref{algln:app_walk} and $W_j^+$ as the walk after Line~\ref{algln:app_walk}. Due to Line~\ref{algln:dist}, if $W_j^-$ is feasible in a particular iteration, then $[W_j^-,y]$ will remain feasible. Hence the only vertices than can possibly have their latency constraints violated in $W_j^+$ are in the orienteering path from $x$ to $y$. Consider any vertex $z$ in the path from $x$ to $y$ returned by the \textsc{Orienteering} function in Line~\ref{algln:app_walk}. If $z\in V(W_j^-)$, then $L(W_j^+,z)\leq r(z)$ because of Line~\ref{algln:dist}. If $z\notin V(W_j^-)$, then $r(z)=\ell(W_j^-) + s_z$ and by Line~\ref{algln:newif}, $r(z)\geq \ell(W_j^-) + d + \ell(y,\mu)$. As $z$ is only visited once in $W_j^+$, $ L(W_j^+,z) = \ell(W_j^+) \leq \ell(W_j^-)+d+\ell(y,a)  \leq r(z)$.

Finally, since each walk starts from $\mu$ and the vertices on each walk have their latency constraints satisfied by that walk, setting $r(\mu)=D$ ensures that the recharging constraint is satisfied for each walk.
\end{proof}

An approximation algorithm for \textsc{Orienteering} can be used in Line~\ref{algln:app_walk} of Algorithm~\ref{alg:greedy}. In our implementation, we used an ILP formulation to solve \textsc{Orienteering}. To improve the runtime in practice, we pre-process the graph before calling the \textsc{Orienteering} solver to consider only the vertices $z$ such that $\ell(x,z)+\ell(z,y)\leq d$. We show in the next section that although the runtime of Algorithm~\ref{alg:greedy} is more than that of Algorithm~\ref{alg:approx_latency}, it can still solve instances with up to $100$ vertices in a reasonable amount of time, and it finds better solutions than the approximation algorithm.

\section{Min Max Weighted Latency Problem}
\label{sec:latency-minmax}
The approximation algorithm and analysis presented previously aid in developing a solution for the multi-robot version of the problem of minimizing the maximum weighted latency. The single robot and multi-robot versions of this problem without recharging constraint are analyzed in~\cite{Alamdari2014} and~\cite{afshani2021approximation} respectively.

The problem of minimizing the maximum weighted latency with recharging constraint is formally defined below.

\begin{definition}[Weighted Latency]
\label{dfn:weighted_latency} Given a graph $G=(\{V\cup\mu\},E)$ with weights $\phi(v)$ for $v\in V$, and a set of walks $\mathcal{W}$, the weighted latency of $v$ is defined as $C(\W,v)=\phi(v)L(\W,v)$. 
\end{definition}

\begin{problem}[Minimizing Maximum Weighted Latency]
\label{pbm:weighted_latency} Given $R$ robots, a graph $G=(\{V\cup \mu\},E)$ with weights $\phi(v)$ for $v\in V$, where $\mu$ is the recharging vertex, and a discharge time $D$ for the robots, find a set of $R$ feasible walks (no robot spends more than $D$ time away from the recharging depot, i.e., $L(W_k,\mu)\leq D$ for each robot $k$) $\W=\{W_1,W_2,\ldots,W_R\}$ such that the cost $\max_v C(\W,v)$ is minimized.
\end{problem}

Without loss of generality, $\phi(v)$ is assumed to be normalized such that $\max_v\phi(v)=1$. In this section we present an algorithm for this problem and relate this problem to Problem~\ref{pbm:min_robs}.

We first present an algorithm for the single robot version of the problem. To the best of our knowledge, no approximation algorithm for the min-max weighted latency problem with recharging constraints exists, even for the single robot version. 

We start with the special case, where all the vertices in the graph have same weight. The following result shows that an approximation algorithm for the RMCCP provides an approximation for this case as well. 

\begin{algorithm}[t]
\DontPrintSemicolon
\caption{\textsc{MinMaxLatencyOneRobot}}
\label{alg:approx_min_max}
\KwIn{Graph $G=(\{V\cup\mu\},E)$ with weights $\phi_v,  \forall v$, discharging time $D$}
	\KwOut{ A walk $S$}
  \vspace{0.2em}
  \hrule
	Let $V_{i}$ be the set of vertices $v$ such that  $ \phi_{\texttt{min}}2^{i-1}\leq \phi_v < \phi_{\texttt{min}}2^{i} $  for $1\leq i \leq \lceil \log_2 \rho \rceil $\;
	
  Let $t = 2^{\lceil \log{\rho}\rceil+1}$ \;
	
  $S_1,S_2,\ldots,S_t \leftarrow \{\}$ \;
	\For {$i=0$ to $\lceil \log \rho \rceil$}
	{Find RMCCP walk for $V_i$ with root $\mu$ and tour constraint $D$ \;\label{algln:rmccp_one_robot}
	 $\{W_{i,0},\ldots,{W_{i,2^i-1}}\} \leftarrow$ Partition RMCCP walk into $2^i$ walks \; \label{algln:walk_break}
	\For {$k=1$ to $t$}{
	 $S_k \gets [S_k,W_{i,j_i},\mu]$, where $j_i = k\mod 2^i$\;
	}
	}
	 $S\leftarrow [S_1,S_2,\ldots,S_t]$\;
\end{algorithm}

\begin{lemma}
\label{lem:one_robot_RMCCP}
    An $\alpha$ approximation algorithm to RMCCP is a $3\alpha$ approximation algorithm to the Problem~\ref{pbm:weighted_latency} with $R=1$ and $\phi(v)=1$ for all $v\in V$.
\end{lemma}
\begin{proof}
The optimal solution to Problem~\ref{pbm:weighted_latency} with $R=1$ and uniform vertex weights consists of cycles rooted at $\mu$, each cycle of length at most $D$. Let the total number of cycles in the optimal solution be $k$. Also, let the total length of these $k$ cycles be $\overline{\texttt{OPT}}_1$ which is equal to the maximum weighted latency. Note that there exists an optimal solution where at most one of the cycles has length less than $D/2$ (otherwise using the metric property, we can merge two cycles of length less than $D/2$ to get one cycle satisfying the recharging constraint and having total length at most $\overline{\texttt{OPT}}_1$). Hence, the total length of the optimal solution is
\begin{equation}
\label{eq:optimalonerobot}
    \overline{\texttt{OPT}}_1 \geq (k-1)D/2.
\end{equation} 

Let the number of cycles returned by the $\alpha$ approximation algorithm for RMCCP be $\alpha k^*$, where $k^*$ is the minimum number of rooted cycles that can cover the graph. Since each cycle has length at most $D$, the total length or the maximum weighted latency of this solution is at most $ D\alpha k^*$. Since $k^*\leq k$,
\begin{align*}
    D\alpha k^*\leq D\alpha k &\leq D\alpha k - D\alpha + \alpha \overline{\texttt{OPT}}_1\\
    & = D\alpha (k-1) + \alpha \overline{\texttt{OPT}}_1\\
    &\leq 2\alpha \overline{\texttt{OPT}}_1 +\alpha \overline{\texttt{OPT}}_1,
\end{align*}
where the second inequality is true because $D\leq \overline{\texttt{OPT}}_1$ (otherwise the TSP solution is optimal and $\alpha$ approximation of RMCCP will return $\alpha$ cycles each of length at most $1.5$ times the length of TSP, resulting in $1.5\alpha$ approximation) and the last inequality follows from~\eqref{eq:optimalonerobot}.
\end{proof}

\begin{proposition}
Algorithm~\ref{alg:approx_min_max} is a $O(\min\{\log n,\frac{\log{D}}{\log{\log{D}}}\}\log{\rho})$ approximation algorithm for the Problem~\ref{pbm:weighted_latency} with $R=1$.
\end{proposition}
\begin{proof}
First we show that the walk returned by Algorithm~\ref{alg:approx_min_max} is feasible. Since RMCCP finds a set of rooted tours for the partition $V_i$ such that no tour has length more than $D$, the recharging constraint is satisfied for each partition $V_i$. The walk $S_k$ is constructed by connecting the walks $W_{i,j}$ via the recharging vertex, hence the feasibility of $S_k$ and therefore $S$ holds (as each $S_k$ starts from the recharging vertex).\\
Let the relaxed vertex weights be given be $\bar{\phi}_i$ and let the vertices with weight $\frac{1}{2^i}$ be denoted by $V_i$. The algorithm constructs a binary walk $[S_1,S_2,\ldots,S_t]$ where $t=\log{\frac{\max_i{\bar{\phi}_i}}{\min_i{\bar{\phi}_i}}}$, and each walk $S_k$ starts at the recharging depot. 

In a partition $V_i$, the algorithm breaks down the walk returned by the approximation algorithm of RMCCP into $2^i$ subwalks $\{W_{i,0},\ldots,{W_{i,2^i-1}}\}$. The vertices of $V_i$ will have a maximum latency $2^i\texttt{OPT}_{G'}$ where $\texttt{OPT}_{G'}$ is the cost of the optimal solution on $G'$. Therefore, an optimal solution to RMCCP on $V_i$ will have a total length at most $2^i\texttt{OPT}_{G'}$ (otherwise, the optimal solution to min-max weighted latency problem has cost more than $\texttt{OPT}_{G'}$). As the solution to RMCCP is partitioned into $2^i$ walks in Line~\ref{algln:walk_break}, length of the walk $W_{i,j}$ is at most $\alpha\texttt{OPT}_{G'}$ where $\alpha$ is the approximation ratio of RMCCP. 

The walk $S_k$ is given by $[W_{0,j_0},\nu,W_{1,j_1},\nu,\ldots,\nu,W_{\log{t},j_{\log{t}}}]$ where $k=j_i (\mod 2^i)$ for $0\leq i < \log{t}$. Since, there are $\log{\rho_{G'}}$ walks in $S_k$, and those walks are connected using the recharging vertex, and $\ell({\nu,v})<\texttt{OPT}_{G'}/2$ for all $v\in V$, the length of the walk $S_k$ is at most $O(\alpha\log{\rho_{G'}})\texttt{OPT}_{G'}$.\\
The cost of the walk $S$ in $G'$ is at most $\texttt{OPT}_{G'}+2\max_k{S_k}$ (see Lemma 5.1 in~\cite{Alamdari2014}). Also $\texttt{OPT}_{G'}\leq\texttt{OPT}_{G}\leq2\texttt{OPT}_{G'}$ (see Lemma 3.2 in~\cite{Alamdari2014}). Using the approximation ratio of RMCCP, the maximum weighted latency of the walk returned by algorithm~\ref{alg:approx_min_max} is $O(\min\{\log n, \frac{\log{D}}{\log{\log{D}}}\}\log{\rho})\texttt{OPT}_G$.
\end{proof}

\begin{algorithm}[ht]
\DontPrintSemicolon
\caption{\textsc{LatencyWalks}}
\label{multiple_latency}
	\KwIn{ Graph $G=(\{V\cup\mu\},E)$, vertex weights $\phi(v), \forall v\in V$, discharging time $D$ and number of robots $R$}
	\KwOut{ A set of $R$ walks $\{W_1,\ldots,W_R\}$ in $G$}
  \vspace{0.2em}
  \hrule
	 $\rho = \max_{i,j} \phi_i/\phi_j$\;
	\If {$\max_{i,j} \phi_i/\phi_j$ is a power of $2$} {$\rho = \max_{i,j} \phi_i/\phi_j+1$ }
 
	 Let $V_{i}$ be the set of vertices of weight $\frac{1}{2^i} < \phi(u) \leq \frac{1}{2^{i-1}} $  for $1\leq i \leq \lceil \log_2 \rho \rceil $ \;
	\If {$R< \log\rho$}{
	\For {$j=1$ to $R$}{
	 Let $G_j$ be a subgraph of $G$ with vertices $V_i$ for  $\lceil\frac{j-1}{R}\log\rho\rceil +1 \leq  i\leq \lceil\frac{j}{R}\log\rho\rceil$ \; \label{algln:Gj}
	 $W_j = $ \textsc{MinMaxLatencyOneRobot}$(G_j)$ \;
}
}
	\If {$R \geq \log\rho$}{
	 Equally space $\lfloor R/\lceil\log\rho\rceil \rfloor$ robots on RMCCP solution of $V_i$ for all $i$ to get  $\{W_1,\ldots,W_{ \lceil\log\rho\rceil \lfloor \frac{R}{\lceil\log\rho\rceil }\rfloor}\}$ \;
	\For {$k=R-\lceil\log\rho\rceil \lfloor \frac{R}{\lceil\log\rho\rceil }\rfloor+1$ to $R$} {
	 Find subset $V_i$ that has the maximum cost with currently assigned robots \;
	 Equally space all the robots on $V_i$ along with robot $k$ to get $W_k$ \;
	}
	}
\end{algorithm}

Algorithm~\ref{alg:approx_min_max} is used as a subroutine in Algorithm~\ref{multiple_latency} to find walks for $R$ robots. The following result characterizes the cost of the solution returned by Algorithm~\ref{multiple_latency}.

\begin{proposition}
Given an instance of Problem~\ref{pbm:weighted_latency}, Algorithm~\ref{multiple_latency} constructs $R$ feasible walks such that the maximum weighted latency of the graph is $O(\min\{\log n,\frac{\log D}{\log \log D}\}\frac{\log \rho}{R})\texttt{OPT}_1$ where $\texttt{OPT}_1$ is the maximum weighted latency of the single optimal walk.
\end{proposition}
\begin{proof}
We first consider the case when $R<\log{\rho}$. The maximum vertex weight in the subgraph $G_j$ constructed at Line~\ref{algln:Gj} of the algorithm will be at most $1/(2^{\frac{j-1}{R}\log\rho})$, whereas the minimum vertex weight in $G_j$ will be at least $1/(2^{\frac{j}{R}\log\rho})$. Hence the ratio of the maximum to minimum vertex weights in $G_j$ will be at most $\rho_j = 2^{\frac{\log\rho}{R}}$. Therefore, the approximation algorithm for one robot will return a walk $W_j$ such that the maximum weighted latency of $W_j$ will be $O(\min\{\log n, \frac{\log D}{\log\log D}\}\log\rho_j)\texttt{OPT}_1^j$. Moreover, $\texttt{OPT}_1^j \leq \texttt{OPT}_1$ and hence if $R<\log \rho$, the maximum weighted latency will be at most $O(\min\{ \log n, \frac{\log D}{\log\log D}\}\frac{\log\rho}{R})\texttt{OPT}_1$.

Now, we consider the case when $R\geq \log{\rho}$. From Lemma~\ref{lem:one_robot_RMCCP}, an $\alpha$ approximation algorithm to RMCCP is a $3\alpha$ approximation algorithm for the vertices in $V_i$ when all the vertex weights in $V_i$ are equal. Since the vertex weights within $V_i$ differ by a factor of $2$ at most, the maximum weighted latency of the solution returned by the approximation algorithm for RMCCP will be $O(\min\{\log n, \frac{\log D}{\log\log D}\})\texttt{OPT}_1$. Dividing these cycles between $\lfloor R/\lceil\log\rho\rceil \rfloor$ robots will decrease the latency by a factor of $O(\lfloor R/\lceil\log\rho\rceil \rfloor)$.
\end{proof}

Note that Algorithm~\ref{multiple_latency} bounds the cost of the solution by a function of the optimal cost of a single robot. This algorithm shows that $R$ robots can asymptotically decrease the weighted latency given by a single walk by a factor of $R$, which is not straightforward for this problem as discussed in Section~\ref{sec:factor_R}. A relation between $\texttt{OPT}_1$ and the optimal weighted latency could result in an approximation ratio for Algorithm~\ref{multiple_latency}, however, we were not able to establish such a relation.

In~\cite{Alamdari2014}, an approximation algorithm for the single robot version of Problem~\ref{pbm:weighted_latency} without recharging constraints is given. This algorithm returns a walk in graph $G_j$ such that the maximum weighted latency of that walk is not more than $(8\log\rho_j + 10)\texttt{OPT}_1^j$, where $\rho_j$ is the ratio of maximum to minimum vertex weights in $G_j$ and $\texttt{OPT}_1^j$ is the optimal maximum weighted latency for one robot in $G_j$. Using this approximation algorithm as a subroutine in Algorithm~\ref{multiple_latency} when no recharging constraints are given results in better approximation ratio given below. 

\begin{proposition}
Given an instance of Problem~\ref{pbm:weighted_latency} with infinite discharge time, Algorithm~\ref{multiple_latency} constructs $R$ walks such that the maximum weighted latency of the graph is not more than $(\frac{8\log\rho}{R}+10)\texttt{OPT}_1$ if $R\leq \log\rho$ and $3\texttt{OPT}_1/\lfloor R/\lceil\log\rho\rceil \rfloor$ otherwise, where $\rho=\max \frac{\phi (v_i)}{\phi (v_j)}\ $ and $\texttt{OPT}_1$ is the maximum weighted latency of the single optimal walk.
\end{proposition}
 
Now we show that if there is an approximation algorithm for Problem~\ref{pbm:weighted_latency}, it can be used to solve Problem~\ref{pbm:min_robs} using the optimal number of robots but with the latency constraints relaxed by a factor $\alpha$. This is referred to as a $(\alpha,1)$-bi-criterion algorithm~\cite{iyer2013submodular} for Problem~\ref{pbm:min_robs}.
\begin{proposition}\label{prop:bicriterion}
If there exists an $\alpha$-approximation algorithm for Problem~\ref{pbm:weighted_latency}, then there exists a $(\alpha,1)$-bi-criterion approximation algorithm for Problem~\ref{pbm:min_robs}.
\end{proposition}
We will need the following lemma relating the two problems to prove the proposition. Given an instance of the decision version of Problem~\ref{pbm:min_robs} with $R$ robots, let us define an instance of Problem~\ref{pbm:weighted_latency} by assigning $\phi(v)= \frac{r_{\texttt{min}}}{r(v)}, \forall v\in V$, where $r_{\texttt{min}}=\min_v r(v)$.
\begin{lemma}
\label{lem:walk}
 An instance of the decision version of Problem~\ref{pbm:min_robs} is feasible if and only if the optimal maximum weighted latency is at most $r_{\texttt{min}}$ for the corresponding instance of Problem~\ref{pbm:weighted_latency}.
\end{lemma}
\begin{proof}
If the optimal set of walks $\W$ has a cost more than $r_{\texttt{min}}$, then $L(\W,v) > r_{\texttt{min}}/\phi(v) =r(v)$ for some vertex $v$. Hence the latency constraint for that vertex is not satisfied and the set of walks $\W$ is not feasible.

If the optimal set of walks $\W$ has a cost at most $r_{\texttt{min}}$, then $ L(\W,v)\phi(v) \leq r_{\texttt{min}}$ for all $v$. Hence, $L(\W,v) \leq r_{\texttt{min}}/\phi(v) = r(v)$. So, the latency constraints are satisfied for all vertices and $\W$ is feasible.
\end{proof}

\begin{proof}[Proof of Proposition~\ref{prop:bicriterion}]
If a problem instance of Problem~\ref{pbm:min_robs} with $R$ robots is feasible, then by Lemma~\ref{lem:walk} the optimal set of walks has a cost at most $r_{\texttt{min}}$. The $\alpha$-approximation algorithm for the corresponding Problem~\ref{pbm:weighted_latency} will return a set of walks $\W$ with a cost no more than $\alpha r_{\texttt{min}}$. Hence, $L(\W,v)\leq \alpha r_{\texttt{min}}/\phi(v) =\alpha r(v)$, for all $v$.

Hence, we can use binary search to find the minimum number of robots for which the $\alpha$-approximation algorithm for the corresponding Problem~\ref{pbm:weighted_latency} results in a latency at most $\alpha r(v)$ for all $v$. This will be the minimum number of robots for which the problem is feasible.
\end{proof}

\section{Simulation Results}
\label{sec:latency-sim}
We now present the empirical performance of the algorithms presented in the previous sections. For the approximation algorithm, we used the orienteering problem as a subroutine to solve the RMCCP problem. The orienteering problems in RMCCP and Algorithm~\ref{alg:greedy} were solved using the ILP formulation from~\cite{letchford2013compact} and the ILP's were solved using the Gurobi solver~\cite{gurobi}.
\subsection{Patrolling an Environment}
The graphs for the problem instances were generated randomly in a real world environment. The scenario represents ground robots monitoring the University of Waterloo campus. Vertices around the campus buildings represent the locations to be monitored and a complete weighted graph was created by generating a probabilistic roadmap to find paths between those vertices. Figure~\ref{fig:environment} shows the patrolling environment. To generate random problem instances of different sizes, $n$ random vertices were chosen from the original graph. The latency constraints were generated uniformly randomly between $\text{TSP}/k$ and $k\text{TSP}$ where $k$ was chosen randomly between $4$ and $8$ for each instance. Here $\text{TSP}$ represents the TSP length of the graph found using the LKH implementation~\cite{helsgaun2000effective}. The robot speed was set as $1$ meters per second and the battery discharge time was set as $D = 3$ hours.

For each graph size, $10$ random instances were created. The proposed approximation algorithm, greedy heuristic algorithm and the orienteering based heuristic algorithm were used to find the walks for each problem instance. The average runtimes of the algorithms are presented in Figure~\ref{fig:runtimes}. As expected, Algorithm~\ref{alg:greedy} is considerably slower than the approximation and simple greedy algorithms due to multiple calls to the ILP solver. However, as shown in Figure~\ref{fig:obj_value}, Algorithm~\ref{alg:greedy} also gives the lowest number of robots required for most of these instances. The trend of the number of robots returned by the greedy algorithm and the orienteering based algorithm shows that the idea of visiting more vertices on the way to the greedily picked vertex works well in practice.  
\begin{figure}
\centering
\includegraphics[width=.7\linewidth, angle=90]{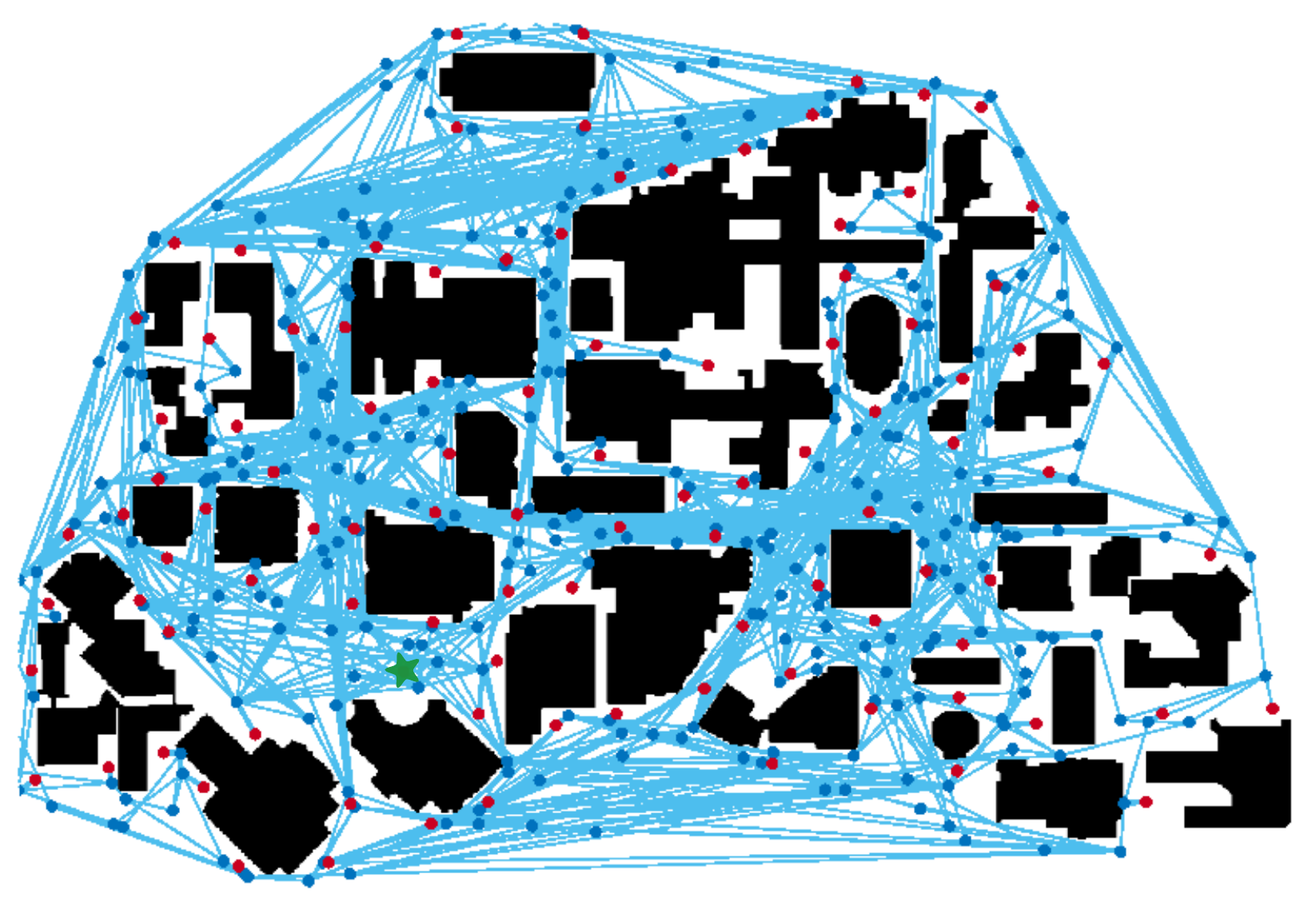}
\caption[Environment used to generate random problem instances]{The environment used for generating random problem instances. The red dots represent the vertices that need to be monitored, while the blue dots indicate the vertices in the Probabilistic Roadmap used to find the shortest paths between the red vertices. The green star represents the charging depot for the monitoring robots.}
\label{fig:environment}
\end{figure}
\begin{figure}
\centering
\includegraphics[width=\linewidth]{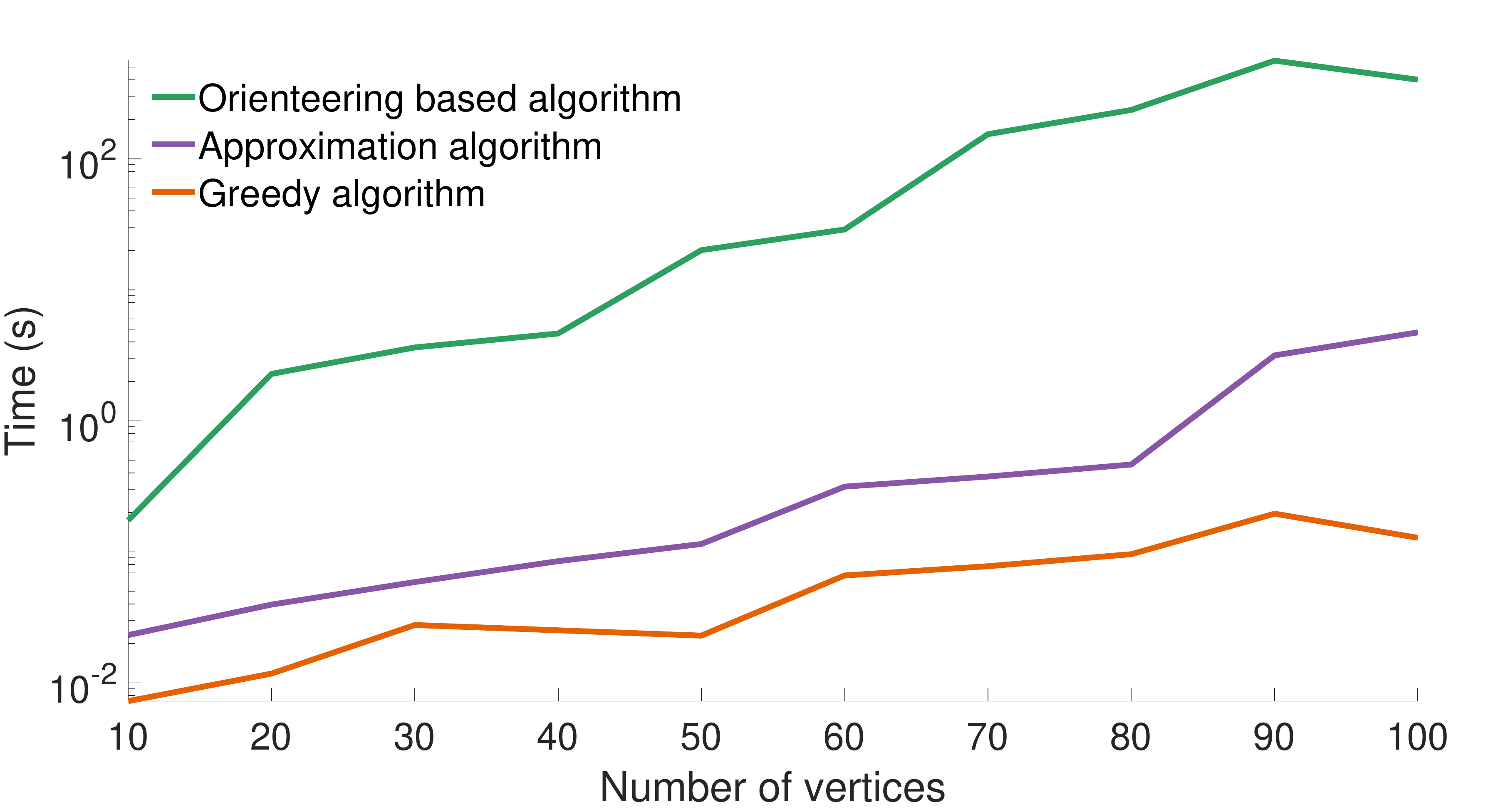}
\caption[Average runtimes of the algorithms]{Average runtimes of the algorithms. The line plot shows the mean run time (on a log scale) over $10$ random instances for each graph size.}
\label{fig:runtimes}
\end{figure}
\begin{figure}
\centering
\includegraphics[width=\linewidth]{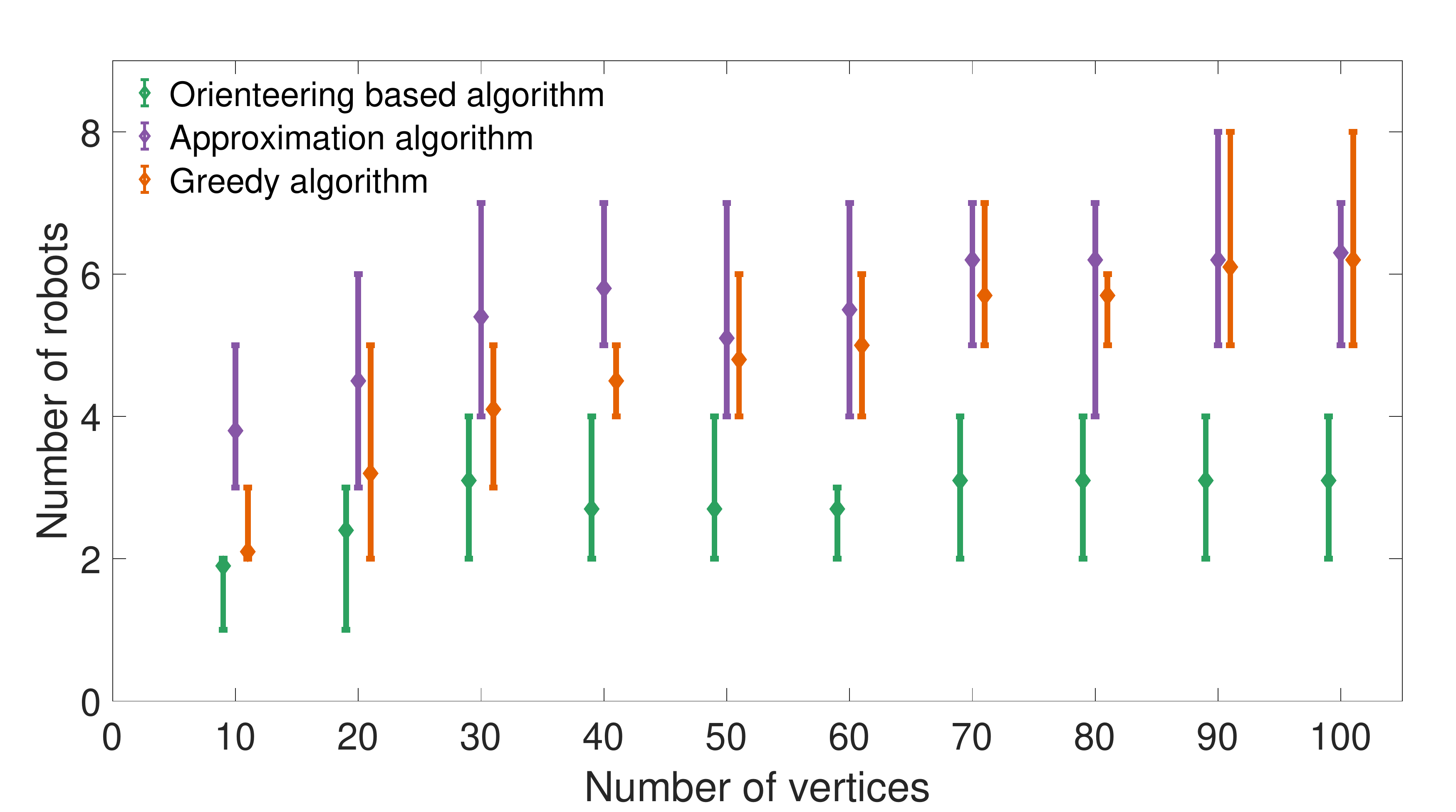}
\caption[Number of robots for each algorithm]{Average number of robots returned by each algorithm. The marker shows the mean over $10$ random instances for each graph size. The error bars show the minimum and maximum number of robots required for a graph size. }
\label{fig:obj_value}
\end{figure}

\subsection{Comparison with Existing Algorithms in Literature}
In~\cite{Drucker2014thesis,drucker2016cyclic} the authors propose an SMT (Satisfiability Modulo Theory) based approach using the Z3 solver~\cite{de2008z3} to solve the decision version of Problem~\ref{pbm:min_robs} without recharging constraints. The idea is to fix an upper bound on the period of the solution and model the problem as a constraint program. The authors also provide benchmark instances for the decision version of the problem. We tested our algorithms on the benchmark instances provided and compare the results to the SMT based solver provided by~\cite{drucker2016cyclic}.

Out of 300 benchmark instances, given a time limit of 10 minutes, the Z3 solver returned $182$ instances as satisfiable with the given number of robots. We ran our algorithms for each instance and checked if the number of robots returned is less than or equal to the number of robots in the benchmark instance. The approximation algorithm satisfied $170$ instances whereas Algorithm~\ref{alg:greedy} satisfied $178$ instances. The four satisfiable instances that Algorithm~\ref{alg:greedy} was unable to satisfy had optimal solutions where the walks share the vertices, and Algorithm~\ref{alg:greedy} returned one more robot than the optimal in all those instances. The drawback of using the constraint program to solve the problem is its poor scalability. It spent an average of $3.76$ seconds on satisfiable instances whereas Algorithm~\ref{alg:greedy} spent $3$ ms on those instances on average. Moreover, on one such instance where Algorithm~\ref{alg:greedy} returned one more robot than the Z3 solver, the Z3 solver spent 194 seconds as compared to $\sim 5$ ms for Algorithm~\ref{alg:greedy}. Note that these differences are for benchmark instances having up to $7$ vertices. As shown in Figure~\ref{fig:runtimes}, Algorithm~\ref{alg:greedy} takes $\sim 100$ seconds for $90$ vertex instances whereas we were unable to solve instances with even $15$ vertices within an hour using the Z3 solver. Hence, the scalability of the Z3 based solver hinders its use for problem instances of practical sizes.

\subsection{Wildfire Monitoring Using Min Max Weighted Latency}
In this section we provide an application example for Problem~\ref{pbm:weighted_latency} and demonstrate the performance of Algorithms~\ref{alg:approx_min_max} and~\ref{multiple_latency}.  Problem~\ref{pbm:weighted_latency}  can be used to monitor an environment where different locations in the environment need to be visited repeatedly and the time between consecutive visits to a particular location depends on the importance of that location. One such application is wildfire surveillance and suppression where Unmanned Aerial Vehicles (UAVs) can be used to detect, monitor and contain fires~\cite{afghah2019wildfire,bailon2022real}. To construct a problem instance, fire hotspot data was used from Canadian Wildland Fire Information System~\cite{CWFIS}. A fire hotspot refers to a pixel in a satellite image that displays a high level of infrared radiation, which signifies the presence of a heat source.
The database contains hotspot data for North America, with different attributes for each hotspot, such as location, modelled rate of spread of fire $\texttt{ros}$ in meters per minute at the hotspot location, approxmiate burned area $\texttt{estarea}$, modelled fire intensity $\texttt{fi}$ in $kW/m$, and others.

We used the database to sample the active hotspots on August 22, 2018, in a 25 kilometers radius around the Gravelly Valley airstrip in California. The vertices in the graph represent the hostpot locations and the vertex weight was calculated as $\texttt{estarea} + \texttt{ros}^2 +\lambda \texttt{fi}$ where $\lambda$ is a weight parameter set to $1/200$ in our experiments. The speed of the UAVs was set as $100$ $km/h$  and the operating time was set as $35$ minutes with the airstrip designated as the refuelling depot. Figure~\ref{fig:forest_fires} displays the monitored environment, including the airstrip and hotspot locations. The hotspot locations are color-coded based on their normalized vertex weights, using the scale in the top right of the image. 

The cycles returned by the Rooted Minimum Cycle Cover Problem in line~\ref{algln:rmccp_one_robot} of Algorithm~\ref{alg:approx_min_max} for the vertex subsets $V_0$ and $V_1$ are shown in Figures~\ref{fig:fire1} and~\ref{fig:fire2} respectively. For the single robot problem, Algorithm~\ref{alg:approx_min_max} returns a walk that traverses the cycles A, B, C, A, B, D in that order. The maximum weighted latency of this walk is compared to a cyclic walk that traverses the cycles A,B,C,D periodically in Table~\ref{tab:table_fire}. The solutions returned for two and three robots by Algorithm~\ref{multiple_latency} are also shown in the table. Their maximum weighted latencies are compared with the solutions where multiple robots are equally spaced on a single robot solution.

\begin{figure}
\centering
\begin{subfigure}{.49\textwidth}
\centering
\includegraphics[width=\linewidth]{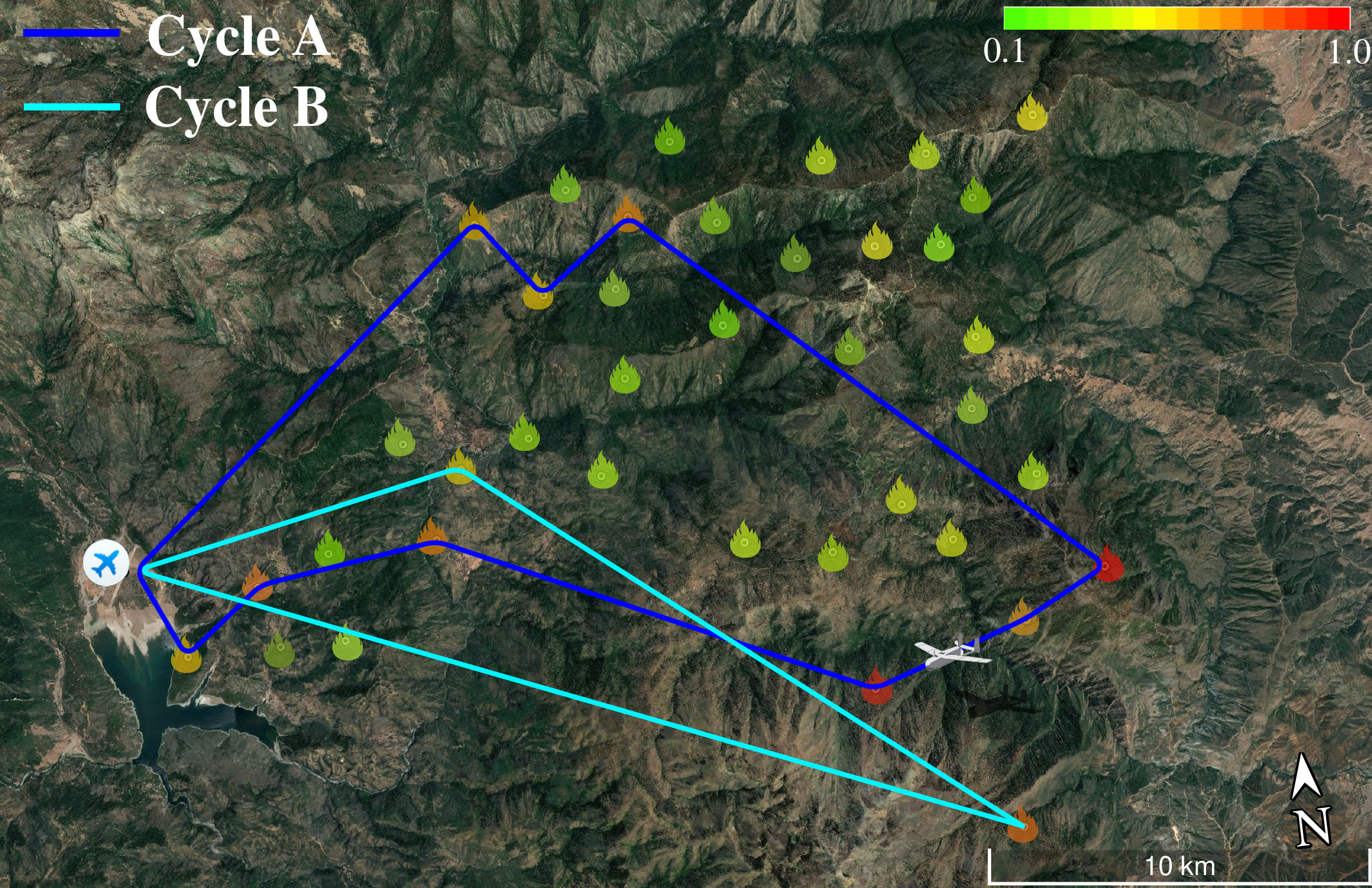}
\caption{Cycles A and B, covering the vertex subset $V_0$ of the vertices with higher vertex weights.}
\label{fig:fire1}
\end{subfigure}\\
\vskip1em
\begin{subfigure}{.49\textwidth}
\centering
\includegraphics[width=\linewidth]{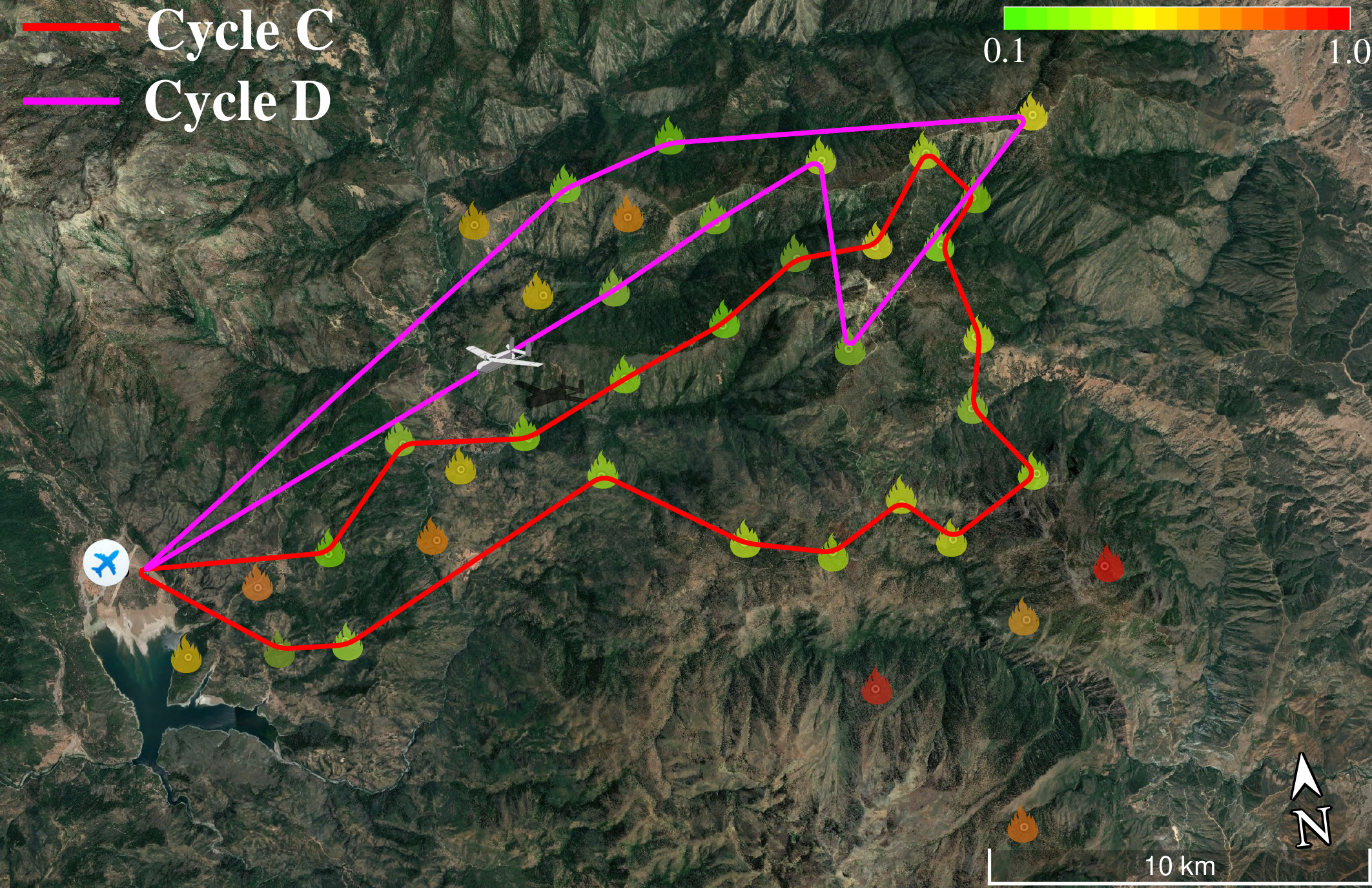}
\caption{Cycles C and D, covering the vertex subset $V_1$ of the vertices with lower vertex weights.}
\label{fig:fire2}
\end{subfigure}
\caption[Walks to continually monitor forest fires]{The environment and solutions for the wildfire monitoring problem. The colored icons represent the fire hotspot locations, with the color of an icon depicting the normalized vertex weight according to the scale shown in the top right of the figure. For a single robot problem, the walk returned by Algorithm~\ref{alg:approx_min_max} is to traverse cycles A,B,C,A,B,D periodically in this order. For the problem with two robots, Algorithm~\ref{multiple_latency} returns walks $\{W_1,W_2\}$ where $W_1$ traverses cycles A and B, and $W_2$ traverses cycles C and D.}
\label{fig:forest_fires}
\end{figure}

\begin{table}
\centering
\begin{tabular}{cccc} \toprule
    {\shortstack{Number \\of UAVs}} & {Algorithm} & {Walks} & {\shortstack{ Max Weighted\\ Latency}} \\ \midrule
    \multirow{2}{*}{1} & {Alg.~\ref{alg:approx_min_max}} & {$W_1=ABCABD$} &  {98.3} \\
    \cmidrule{2-4}
    & {Cyclic solution} & {$W_1= ABCD$} & 132.4\\
    \midrule
    \multirow{4}{*}{2} & \multirow{2}{*}{Alg.~\ref{multiple_latency}} & {$W_1 = AB$} & \multirow{2}{*}{63.4} \\
     &  & {$W_2=CD$} &  \\
     \cmidrule{2-4}
     & \multirow{2}{*}{\shortstack{Equally spaced on\\ cyclic solution}} & {$W_1 = ABCD$} & \multirow{2}{*}{66.2}\\
     & & {$W_2=W_1$ with lag} & \\ 

    \midrule
    \multirow{6}{*}{3} & \multirow{3}{*}{Alg.~\ref{multiple_latency}} & {$W_1 = AB$} & \multirow{3}{*}{33.47} \\
     &  &{$W_3 = CD$}  &  \\
     & & {$W_2=W_1$ with lag} &\\
     \cmidrule{2-4}
     & \multirow{3}{*}{\shortstack{Equally spaced on\\ cyclic solution}} & {$W_1 = ABCD$} & \multirow{2}{*}{44.1}\\
     & & {$W_2=W_1$ with lag} & \\
     & & {$W_3 = W_1$ with lag} & \\
     
    \bottomrule
\end{tabular}
\caption{Solution with different number of UAVs for the wildfire monitoring application. The cycles A,B,C and D are shown in Figure~\ref{fig:forest_fires}. }
    \label{tab:table_fire}
\end{table}

\section{Conclusions and Future Work}

In this paper, we addressed the problem of finding persistent monitoring paths for a team of robots with limited battery capacity that can be recharged at a recharging depot. Our focus was on satisfying latency constraints for locations in the environment. We presented an approximation algorithm and proposed an orienteering-based heuristic approach that provided better solutions than the approximation algorithm on the problem instances tested in simulations. We extended our analysis to propose an algorithm for minimizing the maximum weighted latency given a fixed number of robots. We evaluated our algorithms on different problem instances, including a wildfire monitoring application.

As a future direction, we propose incorporating multiple recharging depots for the robots, which could yield interesting results. Moreover, Algorithm~\ref{multiple_latency} provides a cost bound for the solution as a function of the optimal cost of a single robot. Investigating the relationship between the optimal cost of a single robot and the cost for multiple robots may result in an approximation ratio for Algorithm~\ref{multiple_latency}. In practical scenarios, battery consumption by the robots may be stochastic. Therefore, solving the persistent monitoring problem for these cases may require consideration of the probability of robots running out of charge before recharging at a recharging depot. Finding risk-averse solutions in such scenarios is also an interesting direction for future work.

\bibliographystyle{IEEEtran}
\bibliography{refs}     
\end{document}